\documentclass{article}
\usepackage[left=2cm,top=2cm,right=2cm]{geometry}
\usepackage{microtype}
\usepackage{graphicx}
\usepackage{subfigure}
\usepackage{hyperref}
\hypersetup{
 colorlinks=True,
 linkcolor=blue,
 citecolor=blue,
 urlcolor=blue
 }

\usepackage{amsmath}
\usepackage{amssymb}
\usepackage{mathtools}
\usepackage{amsthm}
\usepackage[capitalize,noabbrev]{cleveref}

\theoremstyle{plain}
\newtheorem{theorem}{Theorem}[section]
\newtheorem{proposition}[theorem]{Proposition}
\newtheorem{lemma}[theorem]{Lemma}
\newtheorem{corollary}[theorem]{Corollary}
\theoremstyle{definition}

\newtheorem{assumption}[theorem]{Assumption}
\theoremstyle{remark}
\newtheorem{remark}[theorem]{Remark}

\usepackage{times}
\usepackage{amsfonts, bm}
\usepackage{graphicx}
\usepackage{amssymb}
\usepackage{pifont}

\usepackage{makecell}
\usepackage[utf8]{inputenc}
\usepackage[T1]{fontenc} 
\usepackage{mathrsfs} 
\usepackage{xspace}
\usepackage{algorithm}
\usepackage{algorithmic}

\usepackage{array}
\usepackage{multirow}
\usepackage{color}
\usepackage[english]{babel}
\usepackage{natbib}
\usepackage{wrapfig}
\usepackage{epstopdf}
\usepackage{url}
\usepackage{bbm}
\usepackage{comment}
\usepackage{booktabs,caption}
\usepackage[flushleft]{threeparttable}
\usepackage{enumitem}
\usepackage{mdframed}
\usepackage[final]{title}

\allowdisplaybreaks

\title{Dataset Reset Policy Optimization for RLHF}
\author{%
    Jonathan D. Chang$^{*}$ \\
     Department of Computer Science\\
    Cornell University\\
    \texttt{jdc396@cornell.edu}\\
    \And
    Wenhao Zhan\thanks{Equal contribution} \\
    Department of Electrical and Computer Engineering\\
    Princeton University\\
    \texttt{wenhao.zhan@princeton.edu}\\
    \And
    Owen Oertell \\
    Department of Computer Science\\
    Cornell University\\
    \texttt{ojo2@cornell.edu}\\
    \And
    Kianté Brantley \\
    Department of Computer Science\\
    Cornell University\\
    \texttt{kdb82@cornell.edu}\\
    \And
    Dipendra Misra\\
    Microsoft Research New York\\
    \texttt{dimisra@microsoft.com}\\
    \And
    Jason D. Lee \\
    Department of Electrical and Computer Engineering\\
    Princeton University\\
    \texttt{jasonlee@princeton.edu}\\  
    \And
    Wen Sun \\
    Department of Computer Science\\
    Cornell University\\
    \texttt{ws455@cornell.edu}\\
}

\newcommand{\pir}{\pi^{\mathrm{SFT}}}

\newcommand{\hpi}{\widehat{\pi}}

\newcommand{\CQ}{C_{\mathrm{Q}}}

\newcommand{\eeval}{\epsilon_{\mathrm{eval}}}
\newcommand{\pis}{\pi^{\star}}
\newcommand{\CKL}{C_{\mathrm{KL}}}

\newcommand{\Ceval}{C_{\mathrm{eval}}}
\newcommand{\bpi}{\overline{\pi}}
\newcommand{\rlhfalg}{DR-PO with NPG update}
\newcommand{\DR}{\mathcal{D}_{\mathrm{R}}}
\newcommand{\DTR}{\mathcal{D}_{\mathrm{TR}}}

\newcommand{\DTRt}{\mathcal{D}_{\mathrm{TR},t}}
\newcommand{\hr}{\widehat{r}}
\newcommand{\rs}{r^{\star}}

\newcommand{\epmd}{\epsilon_{\mathrm{PMD}}}
\newcommand{\emle}{\epsilon_{\mathrm{MLE}}}

\newcommand{\CTR}{C_{\mathsf{TR}}}
\newcommand{\CST}{C_{\mathsf{ST}}}

\newcommand{\hQ}{\widehat{Q}}

\newcommand{\Ac}{\mathcal{A}}

\newcommand{\Dc}{\mathcal{D}}
\newcommand{\Ec}{\mathcal{E}}
\newcommand{\Fc}{\mathcal{F}}

\newcommand{\Hc}{\mathcal{H}}

\newcommand{\Mc}{\mathcal{M}}

\newcommand{\Rc}{\mathcal{R}}
\newcommand{\Sc}{\mathcal{S}}

\newcommand{\Xc}{\mathcal{X}}

\newcommand{\Eb}{\mathbb{E}}

\newcommand{\Pb}{\mathbb{P}}

\newcommand{\eps}{\epsilon}
\newcommand{\KL}{\mathsf{KL}}

\newcommand{\argmax}{\arg\max}
\newcommand{\argmin}{\arg\min}

\newcommand{\unif}{\mathrm{Unif}}
\newcommand{\TO}{\widetilde{\mathcal{O}}}

\newcommand{\rmax}{r_{\mathrm{max}}}
\newcommand{\ind}{\mathbbm{1}}

\newcommand{\Clip}{\mathrm{Clip}}

\newcommand{\BKL}{B_{\mathsf{KL}}}
\newcommand{\CSFT}{C_{\mathsf{SFT}}}

\definecolor{yxc}{RGB}{255,0,0}
\definecolor{yjc}{RGB}{190,0,255}
\definecolor{whz}{RGB}{0,155,0}
\definecolor{mygreen}{RGB}{0,155,155}

\usepackage{enumitem}
\setlist[itemize]{leftmargin=*}
\setlist[enumerate]{leftmargin=*}

\begin{document}

\maketitle

\begin{abstract}
    Reinforcement Learning (RL) from Human Preference-based feedback is a popular paradigm for fine-tuning generative models, which has produced impressive models such as GPT-4 and Claude3 Opus. This framework often consists of two steps: learning a reward model from an offline preference dataset followed by running \emph{online} RL to optimize the learned reward model.  In this work, leveraging the idea of \emph{reset},  we propose a new RLHF algorithm  with provable guarantees. Motivated by the fact that offline preference dataset provides informative states (i.e., data that is preferred by the labelers), our new algorithm, \emph{Dataset Reset Policy Optimization} (DR-PO), integrates the existing offline preference dataset into the online policy training procedure via \emph{dataset reset}: it directly resets the policy  optimizer to the states in the offline dataset, instead of always starting from the initial state distribution. In theory, we show that DR-PO learns to perform at least as good as any policy that is covered by the offline dataset under general function approximation with finite sample complexity. In experiments, we demonstrate that on both the TL;DR summarization and the Anthropic Helpful Harmful (HH) dataset, the generation from DR-PO is better than that from Proximal Policy Optimization (PPO) and Direction Preference Optimization (DPO), under the metric of GPT4 win-rate. Code for this work can be found at \url{https://github.com/Cornell-RL/drpo}.\end{abstract}
    
\section{Introduction}
Reinforcement learning aims at maximizing a cumulative reward function. However, specifying a reward function in practice can be challenging \citep{wirth2017survey}. Reinforcement Learning with Human Feedback (RLHF) has become an effective approach when a reward function does not exist \citep{christiano2017deep}. Operating under a setting where human labelers provide preference-based feedback (e.g., ranking of generations from an RL agent), RLHF learns a reward model and then optimizes the reward model via RL techniques. RLHF has found applications across various domains, including games \citep{macglashan2017interactive, christiano2017deep, warnell2018deep}, large language models (LLMs) \citep{ziegler2019fine, stiennon2020learning, wu2021recursively, nakano2021webgpt, ouyang2022training, glaese2022improving, bai2022training, ramamurthy2022reinforcement, liu2023languages}, and robot learning \citep{brown2019extrapolating, shin2023benchmarks}.

RLHF typically consists of the following two steps: (1) fitting a reward model using a pre-collected offline preference-based dataset (often generated from some pre-trained models and labeled by humans), (2) and learn a policy via online RL (e.g., Proximal Policy Optimization \citep{schulman2017proximal}) to optimize the learned reward model. These two steps are often done separately in the sense that once the reward model is learned, step (2) only optimizes the reward model without ever using the offline preference dataset. Is there any benefit of re-using the offline data during the procedure of optimizing the reward model via online RL? Prior work on hybrid RL \citep{song2022hybrid, ball2023efficient} demonstrated that combining offline data and online data can often significantly boost learning efficiency. Can we achieve a similar boost in learning efficiency for RLHF? 

Towards answering this, we propose an algorithm called Dataset Reset Policy Optimization (DR-PO), operating under the assumption of \emph{being able to reset}, i.e., we can go back to any state and start policy optimization and data collection from that point (as opposed to reseting to initial states). While being able to reset is certainly an assumption, it is naturally satisfied when using RL to fine-tune generative models like language  models and diffusion models \citep{lee2023aligning}. This is because the underlying Markov transitions are simple, known, and deterministic.  For instance, when using RL to optimize text generation, resetting to a state is equivalent to feeding a partial sentence (together with the initial prompt) to the transformer-based policy.   Our algorithm, DR-PO, is a hybrid RL approach that integrates offline data into an online RL procedure: when collecting online data, DR-PO resets the policy optimizer to the states in the offline dataset for exploration. Algorithmically, DR-PO is simple: it iteratively collects a batch of online data by resetting the policy to states in the offline data, performs policy rollouts, and optimizes the policy using the online batch via policy optimization techniques such as Natural Policy Gradient (NPG) \citep{kakade2001natural} or Actor-critic methods (e.g., PPO \citep{schulman2017proximal}). 

While DR-PO is as simple to implement as most of the existing policy optimization algorithms, we demonstrate that DR-PO achieves strong theoretical guarantees under natural assumptions. Specifically, when optimizing a  reward model learned from an offline preference dataset, DR-PO is capable of learning a policy that is at least as good as \emph{any policy} which is covered by the offline data in terms of maximizing the ground truth rewards, and DR-PO achieves this result under general function approximation with finite sample complexity. DR-PO is also computationally tractable since it only requires supervised learning style oracles such as a Maximum Likelihood Estimation (MLE) oracle (for fitting reward models) and a Least Squares Regression oracle (for learning value functions). Thus DR-PO advances the status of the theoretical work on RLHF (see more detailed discussion in \cref{sec:related}). Empirically, we test our approach on two standard RLHF datasets: TL;DR summarization \citep{stiennon2020learning} and Anthropic HH. In TL;DR summarization, we demonstrate that the summaries generated by DR-PO outperform those from PPO and DPO \citep{rafailov2023direct} in terms of GPT4 win-rate. We also show that when transferring the policies trained on TL;DR to the CNN/DailyMail news articles in a zero-shot manner, policies trained via DR-PO again generate summaries that outperform those from PPO and DPO, indicating that dataset reset does not make DR-PO overfit. Finally, we test how DR-PO scales on Anthropic HH \citep{bai2022constitutional} across three different model scales and show that DR-PO scales just as well as PPO while still outperforming baselines.  

Our key contributions can be summarized as follows. 
\begin{itemize}
    \item We propose to use the idea of dataset reset to integrate offline data into online RLHF. Reset is a property that comes for free when optimizing generative models using RL. By leveraging dataset reset, our new algorithm DR-PO achieves strong performance guarantees and offers significant benefits in terms of computation tractability over prior theoretical RLHF works.  
    \item When instantiating PPO as a policy optimizer in DR-PO, we show that our approach can outperform strong baselines PPO and DPO over two standard RLHF benchmarks: TL;DR summarization and Anthropic HH. DR-PO achieves superior empirical performance over PPO without introducing any additional computation or memory overhead to PPO.  
\end{itemize}

\section{Related Work}
\label{sec:related} 
\paragraph{Provably efficient RLHF.} The theoretical investigation on online RLHF started in bandit setting with the notion of dueling bandits \citep{yue2012k, zoghi2014relative, dudik2015contextual}, which aims at identifying the optimal arm with human preference feedback over action pairs. Extending this discussion to tabular MDPs, \citet{novoseller2020dueling} proposes a dueling posterior sampling algorithm that requires computing and sampling from the posterior of the model dynamics and reward function, leading to potential computational inefficiency. Another PAC RLHF algorithm for tabular MDPs is presented by \citet{xu2020preference}. However, this method involves computing complicated bonus terms to guide exploration. Additionally, \citet{pacchiano2021dueling, chen2022human} have designed online RLHF algorithms with provable guarantees by updating a confidence set of the policies iteratively, which, unfortunately, are not practically feasible either. In a more recent study, \citet{zhan2023provableon} tackles the problem of reward-free RLHF. Nevertheless, their algorithm introduces a series of non-convex optimization problems which are challenging to solve. Notably, these works either only focus on tabular MDPs \cite{novoseller2020dueling, xu2020preference, pacchiano2021dueling} or rely on specialized function approximation such as linear parametrization \citep{pacchiano2021dueling, zhan2023provableon} and function classes with small Eluder dimension \citep{chen2022human, wu2023making}, which further restricts their application in practice. In contrast, we focus on the setting where preference-labeled data is only available offline, which is more consistent with the settings considered in applications of fine-tuning language models. Also by using the idea of dataset reset, our algorithm works with function approximation that is much more general than the above prior works. 

The study on theoretical offline RLHF is more limited. \citet{li2023reinforcement} focuses on learning the reward from a human's behavior in dynamic discrete choice models rather than from human preference feedback, and thus, the setting is different. \citet{zhu2023principled} studies PAC algorithms for linear models and \citet{zhan2023provableoff} extends the analysis to general function approximation. However, both of their algorithms are not computationally efficient because they rely on constructing a confidence set for the reward function and solving a constrained maximin problem.

\citet{tiapkin2023demonstrationregularized} studied the setting where high-quality expert demonstrations exist. They use behavior cloning to train a policy using expert demonstrations and then run an Upper-confidence-bound style algorithm to optimize a reward function under a KL regularization to the behavior-cloned policy.  They show that for tabular and linear MDP, the expert demonstrations reduce the sample complexity of online RL. We consider preference-based offline datasets, which may not necessarily come from a high-quality expert, and function approximation that is significantly more general than linear and tabular functions. Note that UCB based algorithms can quickly become computationally intractable beyond tabular and linear settings (e.g., \citet{jiang2016contextual, du2021bilinear}). Our algorithm uses the idea of dataset reset for exploration and does not involve any optimism-based exploration strategy, making it computationally tractable even when dealing with general function approximation. We think that the key idea of dataset reset can also be used in the setting from \citet{tiapkin2023demonstrationregularized} to make their algorithm extend beyond the tabular and linear MDP settings.

\paragraph{Empirical RLHF algorithms.} This work continues the recent literature of RLHF algorithms that perform online RL \citep{zhu2023fine, wu2023pairwise, chang2023learning} to finetune large generative models. There have also been efforts to build on top of DPO \citep{rafailov2023direct} with algorithms such as IPO \citep{azar2011reinforcement} and KTO \citep{kto}. In this paper, our work is complementary to many of these efforts in augmenting RL through the incorporation of dataset resets in online generation. Ideas from this work could directly be applied to existing online RLHF algorithms such as P3O \citep{wu2023pairwise} and APA \citep{zhu2023fine}. Given the recent work \citep{yuan2024self} in incorporating online generations to improve DPO, an offline RLHF method, the idea of dataset resets could also be relevant in this space of hybrid RLHF methods. 

\paragraph{Using reset in RL}   The idea of reset is not new in RL  \citep{kakade2003sample, bagnell2004learning, nair2018overcoming, salimans2018learning, yin2022efficient, uchendu2023jump, silver2016mastering, agarwal2019reinforcement, daume2005learning, daume2009search}. When resetting is available, it helps address exploration and credit assignment problems. In this work, we show that resetting to an offline dataset helps in RLHF. The key challenge in RLHF is that the reward model is learned purely from offline data which may not have a global coverage to the entire state space. Our algorithm incorporates KL regularization to ensure the learned policies do not deviate too much from the offline data so that we do not over-optimize the learned reward model (e.g., reward hacking). While the idea of KL-regularization was also used in prior empirical RLHF works (e.g.,\citet{stiennon2020learning, bai2022training}), we show that by combining the two key ideas, KL regularization and dataset reset, our algorithm achieves strong performance in both theory and practice. We also demonstrate the efficacy of our approach in the application of fine-tuning language models.

\section{Preliminaries}
\paragraph{Markov Decision Processes.} In this paper we consider an episodic time-inhomogeneous Markov Decision Process (MDP) $\Mc$ with state space $\Sc=\{\Sc_h\}_{h=1}^H$, action space $\Ac$ and horizon $H$. Here $\Sc_h$ is the subspace of all states at step $h$. We suppose the states incorporate the information of the current step and thus $\{\Sc_h\}_{h=1}^H$ are mutually disjoint. We assume that every episode begins at the same state $s_1$ and ends at the dummy state $s_{H+1}$, but our analysis can be extended to a random starting state easily. In each episode, at step $h\in[H]$, the agent observes the current $s_h$ and executes an action $a_h$. Then the environment generates a reward $\rs(s_h,a_h)$ (which can be \textit{unobservable} to the agent), and transits to a new state $s_{h+1}$, which is sampled from the transition probability $P(\cdot|s_h,a_h)$. Here we suppose the reward function $\rs:\Sc\times\Ac\mapsto[0,1]$ is bounded, and for any possible trajectory $\tau=(s_h,a_h)_{h=1}^H$, we have $\sum_{h=1}^H\rs(s_h,a_h)\leq\rmax$. Note that when the reward is sparse, $\rmax$ can be much smaller than $H$. 

A policy $\pi:\Sc\to\Delta_{\Ac}$ specifies the action selection probability of the agent conditioned on the current state. Given a policy $\pi$, we define its state-action visitation measure as $d^{\pi}_h(s,a)=\Pb^{\pi}(s_h=s,a_h=a)$  for all $s\in\Sc_h,a\in\Ac,h \in [H]$ where $\Pb^{\pi}(\cdot)$ denotes the distribution of the trajectory when executing policy $\pi$. We will also use $d^{\pi}_h(s)=\sum_{a\in\Ac}d^{\pi}_h(s,a)$ to denote the state visitation measure and $d^{\pi}(\tau)$ to denote the distribution of the trajectory under policy $\pi$. We can further define the associated value functions and Q functions of policy $\pi$ and reward function $r$ as $V^{\pi,r}(s)=\Eb_{\pi}[\sum_{t=h}^H r(s_t,a_t) \mid s_h=s], Q^{\pi,r}(s,a)=\Eb_{\pi}[\sum_{t=h}^H r(s_t,a_t) \mid s_h=s,a_h=a]$ for all $h\in[H],s\in\Sc_h,a\in\Ac$.\footnote{For notation simplicity, we drop the usual subscript $h$ in value functions, as we have assumed state $s$ contains the information of time step $h$. } They characterize the expected cumulative reward under policy $\pi$ starting from a state or a state-action pair. 

We aim to find an $\eps$-optimal policy $\hpi$ with respect to the true reward $\rs$ and a target policy $\pis$ which we denote as some high-quality policy ($\pis$ is not necessarily the globally optimal policy), i.e., $V^{\pis,\rs}(s_1)-V^{\hpi,\rs}(s_1)\leq \eps$. Particularly, we would only utilize common oracles such as Maximum Likelihood Estimator (MLE) and Least Squares Regression (LSR). We also want our algorithms to be able to leverage general function classes beyond linear functions.

\paragraph{RL from Human Feedback (RLHF).} We consider the setting where the true reward $\rs$ is unobservable. Instead, we have access to an offline trajectory-pair dataset $\DR=\{(\tau^{0}_m,\tau^{1}_m,o_m)_{m=1}^M\}$ labeled with human preference, where the trajectories $\tau^{0}_m$ and $\tau^{1}_m$ are i.i.d. sampled from some pre-trained policy $\pir$ (e.g., in NLP tasks, this can be the instruction fine-tuned policy, which is also called supervised fine-tuned (SFT) policy). In this work, we do not explicitly consider the learning procedure of $\pir$, and we assume it is given to us. 
Here $o_m\in\{0,1\}$ characterizes the human preference over the trajectory pairs $(\tau^{0}_m,\tau^{1}_m)$ and we suppose the human preference is modeled by a monotonically increasing link function $\Phi$:
\begin{align*}
\Pb(o=1\mid\tau^{0},\tau^{1})=\Pb(\tau^{1}\succ\tau^0)=\Phi(\rs(\tau^1)-\rs(\tau^0)),
\end{align*}
where we use $\rs(\tau)$ to denote $\sum_{h=1}^H\rs(s_h,a_h)$ for any trajectory $\tau=(s_h,a_h)_{h=1}^H$. A widely-used model is the Bradley-Terry-Luce (BTL) model \citep{bradley1952rank} where the link function is chosen to be the sigmoid function $\sigma(x) = 1/\{1 + \exp(-x)\}$. We will use $\kappa=\frac{1}{\inf_{x\in[-\rmax,\rmax]}\Phi'(x)}$ to measure the non-linearity of the link function $\Phi$, which in turn reflects the hardness of learning the reward model from the human preference.  Given $\DR$, we can learn a reward model $\hr$ using MLE:  \looseness=-1
\begin{align}
	\hr=\argmin_{r\in \mathcal{R}} \sum_{m=1}^M -\log \Pb(o=o_m\mid\tau^{0}_m,\tau^{1}_m; r), \label{eq:rewardmle}
\end{align}
With the BTL model, the above NLL becomes 
\begin{align*}
\ind(o^m = 1)\cdot\log \left(1 + \exp(r(\tau^0_m)- r(\tau^1_m) )\right)+\ind(o^m = 0)\cdot\log \left(1 + \exp(r(\tau^1_m) - r(\tau^0_m) )\right),
\end{align*}
which is a  loss function that has been used in many prior RLHF works\citep{christiano2017deep, stiennon2020learning}. We also assume that we have an \emph{unlabeled dataset} $\DTR=\{\tau_n\}_{n=1}^N$ where $\tau_n$ is i.i.d. sampled from $\pir$. Note that $\DTR$ is unlabeled, so it potentially can be much larger than the human-labeled dataset $\DR$.

\paragraph{The Ability to Reset.}  We consider the setting where we  can \textbf{reset} the system. More formally, given any state $s_h$ at time step $h$, we can reset the RL agent directly to $s_h$ and rollout a policy $\pi$. While this is certainly an assumption, it is satisfied in many important applications, e.g., fine-tuning generative models such as LLMs \citep{ouyang2022training, ramamurthy2022reinforcement, chang2023learning} and Diffusion models \citep{lee2023aligning} with RL.  
In text generation, a state $s_h$ typically means a partial sentence. Resetting from this state would then mean that we feed the partial sentence $s_h$ to a transformer based policy and have it generate new tokens one by one starting from the given partial sentence.  We emphasize that in the RL literature, prior works (e.g., PPO and many RL theoretical works \citep{agarwal2021theory, azar2017minimax, jin2020provably, zhan2022offline}) typically do not assume the ability to reset -- they often assume the agent has to always start from some initial states.  However, when reset is available, it is often a game changer, in both theory \citep{yin2022efficient} and in practice (e.g., AlphaGo \citep{silver2016mastering}).

\section{Dataset Reset Policy Optimization}
We present a meta-algorithm here to provide the details of how we leverage the idea of dataset reset to collect online batch data. We abstract away the policy optimization oracle here to emphasize the novelty of our interaction with the environment for online data collection via dataset reset. 
Once the online batch data is collected, we feed it to a policy optimization oracle, e.g., PG, NPG, Actor-critic methods, or a PPO-style update \footnote{Here we mean the specific actor-critic style policy optimization formulation where clipping is used to ensure small policy update, and critic is learned via GAE, on a given online batch data \citep{schulman2017proximal}.}.  

\begin{algorithm}[htbp]
	\caption{Dataset Reset Policy Optimization (DR-PO)}
	\label{alg:drpo}
	\begin{algorithmic}[1]
		\STATE \textbf{Input}: Preference dataset $\DR$, unlabeled dataset $\DTR$, reward function class $\Rc$,  total number of iterations $T$.
		\STATE \textbf{Initialize}: $\pi^1=\pir$. 
		\STATE Learn a reward model $\hr$ via MLE based on Eq.~\eqref{eq:rewardmle}.
		\FOR{$t=1,\cdots,T$}
			\STATE Initialize an empty online batch $\mathcal{D}_{on}$.
   
            \textcolor{mygreen}{/* Online data collection */}
			\FOR{$n = 1, \cdots N$} 
				\STATE Randomly sample a trajectory in $\DTR$ and a state $s_h$ from it where $h\in[H]$.
				\STATE \textcolor{red}{Reset} $\pi^t$ to $s_h$ and rollout $\pi^t$ to generate trajectory $\{s_{h},a_{h},\dots, s_H, a_H\}$.
				\STATE Add trajectory $\{s_{h'},a_{h'}, \hr(s_{h'},a_{h'}), \ln( \pi^t(a_{h'}|s_{h'}) / \pir(a_{h'}|s_{h'})  )\}_{h'=h}^H$ to $\mathcal{D}_{on}$. \label{line:online_data}
			\ENDFOR
			\STATE \textcolor{red}{Policy update}: $\pi_{t+1} \Leftarrow \text{PO}( \pi^t, \mathcal{D}_{on})$.  \COMMENT {\textcolor{blue}{PG, NPG / TRPO, CPI, Actor-Critic, PPO}}
		\ENDFOR
	\end{algorithmic}
\end{algorithm}

Algorithm~\ref{alg:drpo} summarizes the key idea of dataset reset in DR-PO. The key difference between DR-PO and a more standard policy optimizer is that in DR-PO, for each episode, the policy collects online trajectories via \textbf{resetting to a state randomly sampled from some trajectory in the offline dataset $\DTR$}. In other words, we do not rollout the policy $\pi$ from the initial state $s_1$ as typically done in standard policy optimization algorithms like PG.  
The online data collection procedure collects a batch of online trajectories $\mathcal{D}_{on}$. Note for each online trajectory, we record each state-action pair's reward measured under the learned reward model $\hr$, and also the log ratio of $\pi^t$ and $\pir$ which serves as an empirical estimate of the policy KL divergence, i.e., $\KL( \pi^t(s_{h'} ) || \pir(s_{h'}))$.  Such a KL divergence term can be optionally used as a reward penalty to ensure the learned policies do not deviate too far from $\pir$ so that the reward model $\hr$ stays as a good approximation of the true reward $\rs$ under learned policies' trajectory distributions. We use this KL penalty both in theory and in practice. 

Once the online data is collected, we feed it to a policy optimization oracle $\text{PO}$ for a policy update. A PO oracle can be a PG, NPG, or PPO style update. To be more specific, for a PPO style update procedure, we use $\mathcal{D}_{on}$ to fit a critic for advantage estimation $\widehat{A}(s,a)$\footnote{when using KL penalty, this advantage function measures the advantage under KL regularized reward --- $\hr - \lambda \KL$ with $\lambda \in \mathbb{R}^+$ as coefficient for the KL penalty. } (e.g., via generalized advantage estimation used in PPO), and then update the policy on $\mathcal{D}_{on}$ with the clipping trick: $\pi^{t+1} \Leftarrow \argmax_{\pi} \sum_{s,a\in \mathcal{D}_{on}} \Clip\left( \frac{ \pi(a|s) }{ \pi_t(a|s)} \right)\widehat A(s,a)$. This is the policy update that we use in our experiments. In our theory, we use NPG as the PO oracle. While PPO and NPG are different when it comes to exact implementation, PPO can be understood as a heuristic that approximates NPG for the purpose of being more scalable for large-scale optimization (e.g., the clipping trick induced by PPO is approximately trying to ensure that the new policy does not deviate too much from the old one -- a key property that NPG methods advocated for \citep{kakade2001natural, kakade2002approximately, bagnell2003covariant, schulman2015trust}). 

Implementation wise, with PPO as a PO oracle, given a standard PPO implementation, all we need to do is to feed the policy optimization and GAE oracles in PPO using the online batch of data collected in our way, i.e., $\mathcal{D}_{on}$ collected via dataset reset. Our experiments on two RLHF datasets show that hyperparameters that work well for PPO also work for DR-PO. 

\section{Theoretical Analysis}
\label{sec:theory}
In this section, we analyze the DR-PO (Alg~\ref{alg:drpo}) by instantiating the policy optimization oracle $\text{PO}$ to be a Natural Policy Gradient (NPG) oracle. For completeness, we describe $\text{PO}$ in Algorithm~\ref{alg:npg}, which in high level consists of policy evaluation via least square regression, and then policy update via Mirror Descent style procedure. We leave the detailed full description of the algorithm in Appendix~\ref{sec:algun}. 

In Alg.~\ref{alg:npg}, we use the online data to fit a $Q$ function estimate of the current policy $\pi^t$. Once we learn the critic, we perform policy update via running KL-based Mirror Descent. Note that this step has a closed-form expression for $\pi^{t+1}$:
\begin{align*}
	 	\pi^{t+1}(a|s)\propto\left(\pir(a|s) \right)^{\frac{\eta\lambda}{\eta\lambda+1}}\cdot\left(\pi^t(a|s)\right)^{\frac{1}{\eta\lambda+1}} \cdot \exp \left(\frac{\eta}{\eta\lambda+1}\cdot Q(s,a)\right)
\end{align*}
Note that the KL penalty to $\pir$ in the policy update procedure is important to ensure that $\pi^{t+1}$ does not deviate too much from $\pir$. Also this type of updates ensures that the support of $\pi^{t}(\cdot | s)$ is always a subset of the support of $\pir(s)$ for all state $s$.

\begin{algorithm}[htbp]
	\caption{NPG update for the \text{PO} oracle in Alg.~\ref{alg:drpo}}
	\label{alg:npg}
	\begin{algorithmic}[1]
		\STATE \textbf{Input}: Online dataset $\mathcal{D}_{on}$, the previous policy $\pi^t$, Q function class $\Fc$, regularization parameter $\lambda$, learning rate $\eta$
		\STATE Create an empty regression dataset $\mathcal{D}$.
		\FOR{each (partial) trajectory $\tau$ in $\mathcal{D}_{on}$}
			\STATE Take the first state-action pair $(s_h,a_h)$ in $\tau$ and calculate the total reward $y = \sum_{h'=h}^{H} \hr(s_{h'},a_{h'}) $
			\STATE Add $((s_h,a_h), y)$ to $\mathcal{D}$
		\ENDFOR
		\STATE \textbf{Learn critics}: 
		$$Q=\argmin_{f\in\Fc}\frac{1}{|\Dc|}\sum_{(s,a,y)\in\Dc}\left[\left(f(s,a)-y\right)^2\right].$$
		\STATE \textbf{Policy update}:
      \begin{align}
    		\pi^{t+1}(s)=\argmin_{p\in\Delta(\Ac)}\left\langle -Q(s,\cdot), p\right\rangle \notag+ \lambda\KL(p\Vert\pir(s)) + \frac{1}{\eta}\KL(p\Vert\pi^t(s)), \forall s.
      \end{align}
	\end{algorithmic}
\end{algorithm}

\begin{remark}
Though we mainly focus on the settings where we can reset, when resetting is not possible (e.g., real robotics applications), we can implement the reset by a roll-in and roll-out procedure since we have access to $\pir$: we roll-in $\pir$ to some $s_h$, and then continue by rolling out our policy that is being optimized. This procedure is closely related to the \text{PPO++} algorithm proposed in \citet{chang2023learning}, where the authors empirically demonstrated that it outperforms vanilla PPO on some RLHF benchmarks (but no detailed theoretical investigation). When resetting is available, by directly resetting to the offline data generated by $\pir$, we further reduce computation. 
\end{remark}

\subsection{Theoretical Sample Complexity}
Now we introduce the required assumptions in our analysis.

\paragraph{Function classes.}
We first assume that the reward function class and $Q$ function class are realizable and bounded:
\begin{assumption}[reward function classes]
	\label{ass:reward}
Suppose that we have $\rs\in\Rc$. In addition, assume that $0\leq r(\tau)\leq\rmax$ for all $r\in\Rc$ and trajectory $\tau$.
\end{assumption}
\begin{assumption}[$Q$ function classes]
	\label{ass:Q}
Suppose that we have $Q^{\pi^t,\hr}\in\Fc$ for all $t\in[T]$. In addition, assume that $0\leq f(s,a)\leq \rmax$ for all $f\in\Fc,s\in\Sc,a\in\Ac$.
\end{assumption}

Realizability is a standard assumption used in the theoretical analysis of supervised learning. It is possible to extend our analysis to the setting where model-misspecification exists, and we leave this extension as a future work. 

\paragraph{Concentrability.} Then we assume that $\pir$ can cover the comparator policy $\pis$. In addition, we know the learned policy $\hpi$ is close to $\pir$ in terms of Kl divergence due to the regularizer $\KL(\cdot\Vert\pir)$ in the mirror descent step. Thus, to deal with distribution shift, we also assume $\pir$ can cover the policies which are close to itself:
\begin{assumption}[single-policy concentrability]
	\label{ass:conc}
	Suppose that we have for any $\BKL\geq0$:
	\begin{align*}
&(1)\max_{\tau}\frac{d^{\pis}(\tau)}{d^{\pir}(\tau)}= \CTR <\infty ;\\
&(2)\max_{h\in[H],s\in\Sc_h,a\in\Ac}\frac{d^{\pis}_h(s,a)}{d^{\pir}_h(s,a)}= \CST < \infty; \\
&(3)\max_{\pi\in\Theta(\pir,\BKL),\tau}\frac{d^{\pi}(\tau)}{d^{\pir}(\tau)}= \CSFT(\BKL),
	\end{align*}
 where $\Theta(\pir,\BKL):=\{\pi:\KL(\pi(s)\Vert\pir(s))\leq\BKL,\forall s\in\Sc\}$.
\end{assumption}

Note that in Assumption~\ref{ass:conc} we need $\pir$ to cover $\pis$, both trajectory-wise and state-action-wise. In particular, we always have $\CST\leq\CTR$. Assuming trajectory-wise covering is necessary in RLHF because the human feedback is also trajectory-wise, as shown by the lower bounds in \cite{zhan2023provableoff}.  Intuitively, if the offline data only covers low performance policies' traces, then the learned reward model cannot guarantee to recognize trajectories from a high performance policy during test time (because it has never seen such things in training). 

\begin{remark}
We can indeed relax Assumption~\ref{ass:conc} by leveraging the information in $\Rc$ and $\Fc$, as shown in the discussion in Appendix~\ref{sec:proof-rlhf}.
\end{remark}

\begin{remark}
Note that we have $\CSFT(\BKL)<\infty$ for all $\BKL<\infty$ naturally because $\pi\in\Theta(\pir,\BKL)$ has bounded KL diveregnce with respect to $\pir$.
\end{remark}

Under the above assumptions, we have the following theorem to characterize the suboptimality of $\hpi$ returned by Algorithm~\ref{alg:rlhf}. Recall that $\kappa=\frac{1}{\inf_{x\in[-\rmax,\rmax]}\Phi'(x)}$ measures the non-linearity of the link function $\Phi$.

\begin{theorem}
\label{thm:rlhf}
Suppose Assumption~\ref{ass:reward},\ref{ass:Q},\ref{ass:conc} hold. For any $\delta\in(0,1]$, let 
\footnotesize
\begin{align*}
\emle:=\Theta\left(\sqrt{\frac{\kappa^2}{M}\log\frac{|\Rc|}{\delta}}\right), 
\eeval:=\Theta\left(\sqrt{\frac{\rmax^2}{N}\log\frac{T|\Fc|}{\delta}}\right),
\end{align*}
\normalsize
and set $\eta = \sqrt{\frac{1}{T\rmax^2}}$, then with probability at least $1-\delta$, we have Algorithm~\ref{alg:drpo} with NPG update (Algorithm~\ref{alg:npg}) returns a policy $\hpi$ which satisfies
\begin{align}
\label{eq:subopt}
V^{\pis,\rs}(s_1)-V^{\hpi,\rs}(s_1)\leq \underbrace{(\sqrt{\CTR}+\sqrt{\CSFT(T\rmax/\lambda)})\emle}_{(1)} + \underbrace{2H\sqrt{\CST}\eeval}_{(2)} + \underbrace{\frac{2H^{\frac{3}{2}}\rmax\log\CST}{\sqrt{T}} + \lambda H\log\CST}_{(3)}.
\end{align}
\end{theorem}

Theorem~\ref{thm:rlhf} indicates that the suboptimality of $\hpi$ scales with $\frac{1}{M}$ and $\frac{1}{N}$ polynomially. More specifically, term (1) in Equation~\ref{eq:subopt} measures the estimation error of the reward, (2) is the Q function estimation error and (3) is the optimization error of NPG. We can see that there exists a tradeoff between the estimation error and optimization error. With increasing $T$ and decreasing $\lambda$, the optimization error (3) will decrease while the distirbution shift coefficient $\CSFT$ will become larger, leading to amplified estimation error. In particular, from Theorem~\ref{thm:rlhf}, we can obtain the following sample complexity of DR-PO by setting $T$ and $\lambda$ appropriately:

\begin{corollary}
\label{cor:rlhf}
Suppose Assumption~\ref{ass:reward},\ref{ass:Q},\ref{ass:conc} hold and set 
\footnotesize
\begin{align*}
T = \frac{36H^3\rmax^2\log^2\CST}{\eps^2}, \eta = \sqrt{\frac{1}{T\rmax^2}},\lambda =\frac{\eps}{3H\log\CST},
\end{align*}
\normalsize
then if we have
\begin{align*}
&M =\Omega\left(\frac{\left(\CTR+\CSFT(108H^4\rmax^3\log^3\CST/\eps^3)\right)\kappa^2}{\eps^2}\log\frac{|\Rc|}{\delta}\right),\\
&N = \Omega\left(\frac{H^{2}\rmax^{2}\CST}{\eps^2}\log\frac{T|\Fc|}{\delta}\right),
\end{align*}
we have with probability at least $1-\delta$ that Algorithm~\ref{alg:drpo} with NPG update (Algorithm~\ref{alg:npg}) returns a policy $\hpi$ which satisfies
\begin{align*}
	V^{\pis,\rs}(s_1)-V^{\hpi,\rs}(s_1)\leq\eps.
\end{align*}
\end{corollary}
Theorem~\ref{thm:rlhf} and Corollary~\ref{cor:rlhf} indicate that DR-PO with NPG update is capable of finding an $\eps$-optimal policy with polynomial sample complexity, i.e., $\TO(1/\eps^2)$ labeled trajectory pairs and unlabeled trajectories. Algorithmically, our algorithm does not require pessimism and is model-free, which is much easier and more practical than the pessimistic model-based algorithm proposed in \cite{zhan2023provableoff}.

\begin{remark}
In Theorem~\ref{thm:rlhf} and Corollary~\ref{cor:rlhf} we assume $\Rc$ and $\Fc$ are finite, but our results can be extended to infinite classes directly by replacing $|\Rc| (|\Fc|)$ with their covering numbers.
\end{remark}

\section{Experiments}
We empirically evaluate DR-PO's ability to learn from dataset \textit{resets}. First, we test how well DR-PO is able to both efficiently optimize the reward score as well as minimize the KL-divergence with the reference policy. We also test the generation quality of our resulting policies in terms of Rouge \citep{lin2004rouge} and win rate \citep{rafailov2023direct} against human references measured by GPT4 \citep{achiam2023gpt}. Next, we conduct an ablation study, incrementally relaxing the the proportion of dataset resets in our online data collection to study how sensitive DR-PO is to this hyperparameter. We investigate DR-PO's performance when transferring to another summarization task such as CNN/DailyMail \citep{cnndm}. Finally, we conduct a scaling experiment on Anthropic HH by varying model sizes ranging from 1B to 7B. We find that collecting online generations with dataset resets results in a policy with a better tradeoff between reward optimization and KL-divergence, leading to improved generations over baseline RL algorithms, PPO \citep{schulman2017proximal} and Direct Preference Optimizaion (DPO) \citep{rafailov2023direct}.

\paragraph{Tasks}
We evaluated DR-PO on the TL;DR summarization dataset used in \citet{stiennon2020learning}\footnote{Dataset can be obtained from \url{https://github.com/openai/summarize-from-feedback}} and tested scaling performance on the Anthropic Helpful Harmful (HH) task \citep{bai2022constitutional}. For TL;DR, a model is trained to generate summaries of online Reddit posts guided by human preference data. The task consists of two datasets: one with human reference summaries and another with preference data. Following the standards set by both \citet{stiennon2020learning} and \citet{rafailov2023direct}, we train our reward models and DPO baseline on the preference dataset while performing online RL (for PPO and DR-PO) on the human reference dataset. We set the maximum context length to be 512 and the maximum generation length to be 53, ensuring that it is possible to generate all references in the dataset. For Anthropic HH, the model is asked to respond to a dialogue sequence in a helpful, harmless manner. We follow much of design choices from TRLx\footnote{\url{https://github.com/CarperAI/trlx}} for dataset processing, context length, and generation length. For more details about the dataset, please see \cref{app:details}

\paragraph{Evaluation}
To test the performance of DR-PO against our baselines we evaluate each method by its tradeoff between reward model score and KL-divergence with the reference policy, testing the effectiveness of the algorithm in optimizing the regularized RLHF objective. Furthermore, we compute the Rouge score and GPT4 win rate to evaluate the generation quality of our resulting policies. Note for our win rate calculation, we report the win rate of a randomly sampled subset (10\%) of the test set for a total of 600 samples. Please see \cref{app:winrate} for the prompt used to query GPT4 as well as an example response. When evaluating the on CNN/DailyMail we make use of the constructed preference dataset from \citet{stiennon2020learning} and for training a supervised finetuned model, we use HuggingFace's dataset version 2.0.0\footnote{\url{https://huggingface.co/datasets/cnn_dailymail}}.

\paragraph{Methods}
We instantiate DR-PO by using PPO style policy optimization \citep{schulman2017proximal} as the policy optimizer (PO in Algorithm~\ref{alg:drpo}). First for TL;DR, we maintain the same pretrained LLM and supervised finetuned model for all of our experiments. For supervised finetuning, we trained a Pythia 2.8B\footnote{HuggingFace Model Card: \url{EleutherAI/pythia-2.8b-deduped}} \citep{biderman2023pythia} parameter model for 1 epoch over the dataset with human references as labels. Similarly for the reward model, we trained a Pythia 2.8B parameter model for 1 epoch over the preference labeled dataset. Then, for DPO, PPO, and DR-PO, we trained our policy and critic with low rank adapters (LoRA) \citep{hu2022lora} on top of our supervised finetuned (SFT) model and our reward model (RM) respectively. Finally for our scaling experiments for Anthropic HH, we trained Pythia 125M, 1B, and 6.9B parameter models for 1 epoch over the HH dataset for both SFT and RM training. Please see \cref{app:details} for details and \cref{ssec:reset_pseudocode} for pseudocode to implement resets.

\begin{table}[t]
    \centering
    \begin{tabular}{@{}lcccccc@{}}
        \toprule
        \textbf{Algorithms} & \multicolumn{6}{c}{\textbf{TL;DR Summarization}} \\
        \cmidrule(lr){2-7}
        & Win Rate & RM Score & KL$(\pi||\pi_{ref})$ & Rouge 1 & Rouge 2 & RougeL \\
         & ($\uparrow$) & ($\uparrow$) & ($\downarrow$) & ($\uparrow$) & ($\uparrow$) & ($\uparrow$)  \\
        \cmidrule(lr){1-1}\cmidrule(lr){2-7}
        \texttt{SFT}   & 31.6 $\pm$ 0.2\%          & -0.51 $\pm$ 0.04         & -                         & 32.17 $\pm$ 1.01         & 12.27 $\pm$ 0.67          & 24.87 $\pm$ 1.22 \\
        \texttt{DPO}   & 52.6 $\pm$ 0.4\%          & -                        & 37.33 $\pm$ 2.01          & 30.03 $\pm$ 3.23         & 7.93 $\pm$ 1.02           & 22.05 $\pm$ 0.83 \\
        \texttt{PPO}   & 62.3 $\pm$ 2.5\%          & 1.17 $\pm$ 0.13          & \textbf{16.32 $\pm$ 1.46} & \textbf{33.73 $\pm$ 2.34} & \textbf{11.97 $\pm$ 0.91} & 24.97 $\pm$ 1.03 \\
        \texttt{DR-PO} & \textbf{70.2 $\pm$ 1.7\%} & \textbf{1.52 $\pm$ 0.09} & 16.84 $\pm$ 0.83          & 33.68 $\pm$ 1.78         & 11.90 $\pm$ 0.06          & \textbf{25.12 $\pm$ 0.76}\\
        \bottomrule 
    \end{tabular}%
    \caption{\textbf{TL;DR Summarization Results:} Our RM Score is under our trained preference reward model and the win rate is evaluated by GPT4. All evaluated policies except for \texttt{SFT} are models with LoRA adapters. We present results across 3 seeds.}
    \label{tbl:tldr_results}
\end{table}
\begin{figure}[t]
    \centering
    \includegraphics{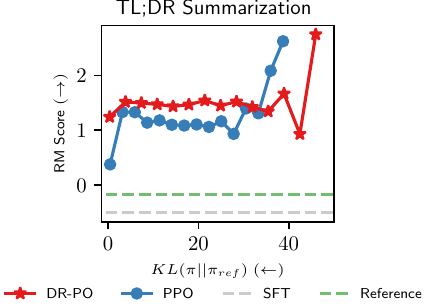}
    \caption{\textbf{Reward vs KL-Divergence Frontier:} Plotting the regularized optimization tradeoff between DR-PO and our baselines over the entire test set. DR-PO is able to achieve a much better tradeoff by learning higher reward generations with lower KL. The average reference and SFT scores under the \texttt{RM} are shown as dashed lines.}
    \label{fig:tldr_pareto}
\end{figure}

\subsection{How well can DR-PO optimize the RLHF objective?}

\cref{tbl:tldr_results} compares DR-PO against PPO, DPO, and supervised finetuning. The KL-regularized reward optimization broadly used in RLHF as well as analyzed in \cref{sec:theory} balances reward exploitation and deviation from a reference policy. When computing the KL-divergence, we use our SFT policy as our reference policy for all our methods. Notably, DR-PO scores a higher RM value over the test set over all baselines with a slightly larger KL discrepancy than PPO. We also see that with GPT4 win rate, DR-PO achieves the highest preference over human references showcasing the benefit of learning from resets. \cref{fig:tldr_pareto} plots a more detailed frontier of the reward and KL tradeoff for DR-PO and PPO. We generate this plot by binning the test scores according to KL. We see that for most KL values, DR-PO is able to achieve a higher score than PPO.

\subsection{Analysis of Dataset Reset Proportion}
\begin{table}[h]
    \centering
    \begin{tabular}{@{}lccc@{}}
        \toprule
        \textbf{Algorithms} & Win Rate & RM Score & KL$(\pi||\pi_{ref})$  \\
         & ($\uparrow$) & ($\uparrow$) & ($\downarrow$) \\
        \cmidrule(lr){1-1}\cmidrule(lr){2-4}
        \texttt{PPO}  & 60.7\% & 1.14 & 15.08 \\
        \texttt{DR-PO} ($\beta=0.25$) & 61.7\% & 1.28 & 14.77 \\
        \texttt{DR-PO} ($\beta=0.5$) & 66.5\% & 1.28 & 15.63 \\
        \texttt{DR-PO} ($\beta=0.75$) & 64.3\% & 1.25 & \textbf{14.32} \\
        \texttt{DR-PO} ($\beta=1.0$) & \textbf{68.5\%} & \textbf{1.47} & 16.65 \\
        \bottomrule 
    \end{tabular}%
    \caption{\textbf{DR-PO Ablation of Datset Reset Proportion:} Our RM Score is under our trained preference reward model and the Win Rate is evaluated by GPT4. $\beta$ represents the proportion of online data generated from dataset resets with 1.0 being all generations are from resets and 0.0 being PPO (i.e., always reset to initial prompts). The values on this table are across one seed.}
    \label{tbl:tldr_beta_ablation}
\end{table}
\begin{figure}[t]
    \centering
    \includegraphics{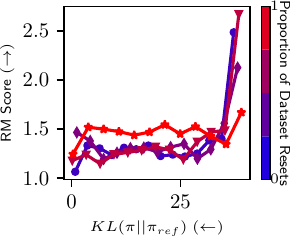}
    \caption{\textbf{Ablation of Dataset Reset:} Plotting the RM score and KL-Divergence tradeoff as a function of dataset reset proportion. \textcolor{blue}{Blue} represents no mixing while \textcolor{red}{red} represents every online generation starting from a reset.}
    \label{fig:tldr_beta_ablation}
\end{figure}
Next, we investigate how sensitive DR-PO is to the amount of dataset resets done during online generation. We define $\beta$ as the proportion of generations in a given online batch of generations with dataset resets. More specifically, our main results are with $\beta=1.0$ which translates to all generations during online training of DR-PO starting from a randomly sampled reset from the human references. Note that a $\beta$ value of 0 recovers the baseline PPO (e.g., all generations start from initial prompts). \cref{tbl:tldr_beta_ablation} shows the expected RM score, KL, and win rate of DR-PO as we increase the mixing proportion from 0\% (PPO) to 100\% (DR-PO) after 2 epochs of training. Notably, even with a small amount of dataset resets DR-PO is able to learn higher scoring generations with a lower KL than PPO. Moreover, we see that DR-PO with any amount of reference resets leads to higher win rates than PPO. \cref{fig:tldr_beta_ablation} plots the RM score/KL-divergence frontier of our learned policies on the test set. Note that DR-PO is robust to the amount of dataset resets in optimizing the regularized RLHF objective. Finally, supporting our analysis from \cref{sec:theory}, DR-PO generally performs better the more online data we gather from resets with a 100\% reset proportion performing the best.

\subsection{DR-PO Transfer Performance}
Finally, we investigate DR-PO's ability to do zero-shot transfer to another summarization task, ensuring that learning a policy by reseting from human references does not diminish the generalization observed with PPO in \citet{stiennon2020learning}. Specifically, we investigate whether leveraging human references on TL;DR has the unintended consequence of overfitting to the specific dataset rather than learning more generally to summarize. For our baselines, we test the zero-shot capabilities of both PPO and DPO as well as report the performance of a supervised finetuned policy on CNN/DailyMail using the same base model, Pythia 2.8B. \cref{tbl:cnn_transfer} demonstrates DR-PO's zero-shot capabilities, being the only policy to outperform a supervised finetuned model on all metrics. Therefore, we see that learning from resets not only improves RLHF on the training task but also the zero-shot transfer performance to another summarization task. 

\begin{table}[tb]
    \centering
    \begin{tabular}{@{}lcccc@{}}
        \toprule
        \textbf{Algorithms} & \multicolumn{4}{c}{\textbf{CNN/DM Summarization}} \\
        \cmidrule(lr){2-5}
        & Win Rate & Rouge 1 & Rouge 2 & RougeL \\
         & ($\uparrow$) & ($\uparrow$) & ($\uparrow$) & ($\uparrow$)  \\
        \cmidrule(lr){1-1}\cmidrule(lr){2-5}
        \texttt{SFT} (CNN/DM)  & 10.5\% & 25.60 & 12.27 & 19.99 \\
        \midrule
        \texttt{DPO}  & 6.0\% & 20.71 & 9.47  & 15.70 \\
        \texttt{PPO}  & 8.5\% & 23.62 & 12.29 & 18.56 \\
        \texttt{DR-PO} & \textbf{12.0}\% & \textbf{29.53} & \textbf{15.36} & \textbf{22.88} \\
        \bottomrule 
    \end{tabular}%
    \caption{\textbf{ Zero-shot transfer to CNN/DM:} the Win Rate is evaluated by GPT4.}
    \label{tbl:cnn_transfer}
\end{table}

\subsection{DR-PO Scaling Performance on Anthropic HH}
\cref{fig:anthropic_scaling} shows DR-PO's performance across different model scales on Anthropic HH task. Specifically we tested three model sizes: 125M, 1B, and 6.9B. We specifically trained on the Pythia models \citep{biderman2023pythia} using TRLx\footnote{\url{https://github.com/CarperAI/trlx}}. We kept the decoding to be the same across all methods here with a sampling temperature of 0.01 as \citet{rafailov2023direct} showed that DPO performed best with greedier sampling. We see that both SFT and DPO showed similar scaling performance gains with PPO and DR-PO scaling better from 1B to 6.9B parameters. \cref{fig:anthropic_scaling} shows that DR-PO has similar scaling improvements as PPO, but performs strictly better and produces generations that are more preferred than those from all of our baselines. 
\begin{figure}[h]
    \centering
    \includegraphics[width=0.5\textwidth]{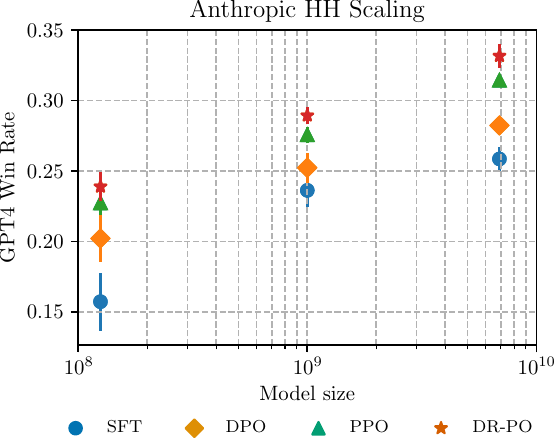}
    \caption{\textbf{Scaling on Anthropic HH:} The GPT4 win rate of DR-PO when tested across 3 model scales: 125M, 1B, and 6.9B. Reported winrates are mean and std across 3 seeds.}
    \label{fig:anthropic_scaling}
\end{figure}

\section{Conclusion}
We present DR-PO, a provably efficient algorithm that exploits a generative model's ability to reset from offline data to enhance RLHF from preference-based feedback. Both in theory and in practice, we demonstrate the effectiveness of incorporating dataset resets into online RL. While in our experiments we specifically demonstrate dataset resets on a PPO style policy optimizer, the idea of dataset reset is both general and simple to implement into any online data collection component of other RL algorithms. We leave it to exciting future work to test the full capabilities of dataset resets in other RLHF methods.

\section*{Acknowledgements}
Wen Sun acknowledges funding from NSF IIS-2154711, NSF CAREER 2339395, and Cornell Infosys Collaboration. Jonathan Chang is supported by LinkedIn under the LinkedIn-Cornell Grant. Kiante Brantley is supported by NSF under grant No. 2127309 to the Computing Research Association for the CIFellows Project.
\bibliographystyle{apalike}

\clearpage
\appendix
\section{DR-PO with NPG}
\label{sec:algun}
\begin{algorithm}[htbp]
	\caption{\textbf{\rlhfalg}}
	\label{alg:rlhf}
	\begin{algorithmic}[1]
		\STATE \textbf{Input}: labeled preference dataset $\DR$, unlabeled dataset $\DTR$, reward function class $\Rc$, Q function class $\Fc$, regularization parameter $\lambda$, stepsize $\eta$, total number of iterations $T$.
		\STATE \textbf{Initialize}: $\pi^1=\pir$. 
		\STATE Learn a reward model $\hat r$ via MLE based on Eq.~\eqref{eq:rewardmle}.
		
		\STATE Let $N_0\gets \frac{N}{T}$. Partition $\DTR$ into $\{\DTRt:=\{\tau^{t,n}\}_{n=1}^{N_0}\}_{t\in[T]}$ with an equal size.

        \textcolor{mygreen}{/* Policy Evaluation with Dataset Reset */}
		\FOR{$t=1,\cdots,T$} 
		\FOR{$n=1,\cdots,N_0$}
        \STATE Sample $h\sim\unif([H])$ and pick the state at step $h$ from $\tau^{t,n}$, denoted by $s^{t,n}_h$.
		\STATE Take action $a^{t,n}_{h}\sim(\frac{1}{2}\pir(s^{t,n}_{h})+\frac{1}{2}\pi^t(s^{t,n}_{h}))$ and then execute $\pi^t$ to step $H$. 
		\STATE Denote the trajectory by $(s_h,a_h,\cdots,s_{H},a_{H})$ and let $y^{t,n}_{h}=\sum_{h'=h}^H\hr_{h'}(s_{h'},a_{h'})$. 
		\STATE Add $(s^{t,n}_{h},a^{t,n}_{h},y^{t,n}_h)$ into $\Dc_{t}$.
		\ENDFOR
		
		\STATE Compute $$\hQ^t=\argmin_{f\in\Fc}L_{\Dc_{t}}(f):=\frac{1}{N_0}\sum_{(s,a,y)\in\Dc_{t}}\left[\left(f(s,a)-y\right)^2\right].$$ 

        \textcolor{mygreen}{/* NPG Update */}
		\STATE Compute for all $s\in\Sc$:
		$$\pi^{t+1}(s)=\argmin_{p\in\Delta(\Ac)}\left\langle -\hQ^t(s,\cdot), p\right\rangle + \lambda\KL(p\Vert\pir(s)) + \frac{1}{\eta}\KL(p\Vert\pi^t(s)).$$  
		\ENDFOR
		
		\STATE \textbf{Output:} $\hpi=\unif(\{\pi^t\}_{t=1}^T)$.
	\end{algorithmic}
\end{algorithm}

\section{Proof of Theorem~\ref{thm:rlhf}}
\label{sec:proof-rlhf}
First we relax the single-policy concentrability in Assumption~\ref{ass:conc} to the following assumptions.
\begin{assumption}[single-policy concentrability w.r.t. the reward class]
	\label{ass:reward-conc}
	Suppose that we have:
	\begin{align*}
		\max\Biggl\{0,~\sup_{r\in\Rc}\frac{\Eb_{\tau^0 \sim d^{\pis},\tau^1 \sim d^{\pir}}[\rs(\tau^0)-\rs(\tau^1)-r(\tau^0)+r(\tau^1)]}{\sqrt{\Eb_{\tau^0 \sim d^{\pir},\tau^1 \sim d^{\pir}}\big[|\rs(\tau^0)-\rs(\tau^1)-r(\tau^0)+r(\tau^1)|^2\big]}}\Biggr\}\leq C_r(\Rc).
	\end{align*}
\end{assumption}
\begin{assumption}[single-policy concentrability w.r.t. Q function class]
	\label{ass:Q-conc}
	Suppose that we have for all $t\in[T]$:
	\begin{align*}
	\sup_{h\in[H],f\in\Fc,\bpi\in\{\pi^t,\pis\}}\frac{\left|\Eb_{s\sim d^{\pis}_h,a\sim\bpi(s)}\left[f(s,a)-\hQ^{\pi^t,\hr}(s,a)\right]\right|}{\sqrt{\Eb_{s\sim d^{\pir}_h,a\sim\left(\frac{1}{2}\pir(s)+\frac{1}{2}\pi^t(s)\right)}\left[\left(f(s,a)-\hQ^{\pi^t,\hr}(s,a)\right)^2\right]}}\leq\Ceval(\Fc)
	\end{align*}
\end{assumption}

\begin{assumption}[single-policy concentrability w.r.t. KL divergence]
	\label{ass:KL-bound}
	Suppose that we have:
	\begin{align*}
		\sum_{h=1}^H\Eb_{s_h\sim d^{\pis}_h}\left[\KL(\pis(s_h)\Vert \pir(s_h))\right]\leq\CKL.
	\end{align*}
\end{assumption}

\begin{assumption}[Concentrability w.r.t. bounded KL-diveregnce policies]
	\label{ass:KL-bound}
	Suppose that we have for any $\BKL\geq0$:
	\begin{align*}
			\max\Biggl\{0,~\sup_{r\in\Rc,\pi\in\Theta(\pir,\BKL)}\frac{\Eb_{\tau^0 \sim d^{\pi},\tau^1 \sim d^{\pir}}[\rs(\tau^0)-\rs(\tau^1)-r(\tau^0)+r(\tau^1)]}{\sqrt{\Eb_{\tau^0 \sim d^{\pir},\tau^1 \sim d^{\pir}}\big[|\rs(\tau^0)-\rs(\tau^1)-r(\tau^0)+r(\tau^1)|^2\big]}}\Biggr\}\leq C_s(\Rc,\BKL).
	\end{align*}
\end{assumption}

Note that from Cauchy-Schwartz inequality we have the following proposition:
\begin{proposition}
\label{prop:relax}
We have
\begin{align*}
&C_r(\Rc)\leq\sqrt{\max_{\tau}\frac{d^{\pis}(\tau)}{d^{\pir}(\tau)}},C_s(\Rc,\BKL)\leq\sqrt{\max_{\pi\in\Theta(\pir,\BKL),\tau}\frac{d^{\pi}(\tau)}{d^{\pir}(\tau)}}\\
&\Ceval(\Fc)\leq\sqrt{2\cdot\max_{h\in[H],s\in\Sc_h,a\in\Ac}\frac{d^{\pis}_h(s,a)}{d^{\pir}_h(s,a)}},\\
&\CKL\leq H\log\left(\max_{h\in[H],s\in\Sc_h,a\in\Ac}\frac{d^{\pis}_h(s,a)}{d^{\pir}_h(s,a)}\right).
\end{align*}
\end{proposition}
\begin{proof}
First from Cauchy-Schwartz inequality, we have
\begin{align*}
\Eb_{\tau^0 \sim d^{\pis},\tau^1 \sim d^{\pir}}[\rs(\tau^0)-\rs(\tau^1)-r(\tau^0)+r(\tau^1)]\leq\sqrt{\Eb_{\tau^0 \sim d^{\pis},\tau^1 \sim d^{\pir}}[\left|\rs(\tau^0)-\rs(\tau^1)-r(\tau^0)+r(\tau^1)\right|^2]}.
\end{align*}
Therefore we have
\begin{align*}
C_r(\Rc)\leq\sqrt{\sup_{r\in\Rc}\frac{\Eb_{\tau^0 \sim d^{\pis},\tau^1 \sim d^{\pir}}[\left|\rs(\tau^0)-\rs(\tau^1)-r(\tau^0)+r(\tau^1)\right|^2]}{\Eb_{\tau^0 \sim d^{\pir},\tau^1 \sim d^{\pir}}\big[|\rs(\tau^0)-\rs(\tau^1)-r(\tau^0)+r(\tau^1)|^2\big]}}\leq\sqrt{\max_{\tau}\frac{d^{\pis}(\tau)}{d^{\pir}(\tau)}}.
\end{align*}
The bound for $C_s$ follows the same arguments.

Similarly, we have:
\begin{align*}
\left|\Eb_{s\sim d^{\pis}_h,a\sim\bpi(s)}\left[f(s,a)-\hQ^{\pi^t,\hr}(s,a)\right]\right|\leq\sqrt{\Eb_{s\sim d^{\pis}_h,a\sim\bpi(s)}\left[\left|f(s,a)-\hQ^{\pi^t,\hr}(s,a)\right|^2\right]}.
\end{align*}
Therefore we have
\begin{align*}
\Ceval(\Fc)\leq\sqrt{\sup_{h\in[H],f\in\Fc,\bpi\in\{\pi^t,\pis\}}\frac{\Eb_{s\sim d^{\pis}_h,a\sim\bpi(s)}\left[\left|f(s,a)-\hQ^{\pi^t,\hr}(s,a)\right|^2\right]}{\Eb_{s\sim d^{\pir}_h,a\sim\left(\frac{1}{2}\pir(s)+\frac{1}{2}\pi^t(s)\right)}\left[\left(f(s,a)-\hQ^{\pi^t,\hr}(s,a)\right)^2\right]}}.
\end{align*}
Note that we have
\begin{align*}
\sup_{h\in[H],f\in\Fc}\frac{\Eb_{s\sim d^{\pis}_h,a\sim\pis(s)}\left[\left|f(s,a)-\hQ^{\pi^t,\hr}(s,a)\right|^2\right]}{\Eb_{s\sim d^{\pir}_h,a\sim\left(\frac{1}{2}\pir(s)+\frac{1}{2}\pi^t(s)\right)}\left[\left(f(s,a)-\hQ^{\pi^t,\hr}(s,a)\right)^2\right]}\leq\max_{h\in[H],s\in\Sc_h,a\in\Ac}2\cdot\frac{d^{\pis}_h(s,a)}{d^{\pir}_h(s,a)}.
\end{align*}
On the other hand, we know
\begin{align*}
&\sup_{h\in[H],f\in\Fc}\frac{\Eb_{s\sim d^{\pis}_h,a\sim\pi^t(s)}\left[\left|f(s,a)-\hQ^{\pi^t,\hr}(s,a)\right|^2\right]}{\Eb_{s\sim d^{\pir}_h,a\sim\left(\frac{1}{2}\pir(s)+\frac{1}{2}\pi^t(s)\right)}\left[\left(f(s,a)-\hQ^{\pi^t,\hr}(s,a)\right)^2\right]}\\
&\qquad\leq\max_{h\in[H],s\in\Sc_h}2\cdot\frac{d^{\pis}_h(s)}{d^{\pir}_h(s)}\leq\max_{h\in[H],s\in\Sc_h,a\in\Ac}2\cdot\frac{d^{\pis}_h(s,a)}{d^{\pir}_h(s,a)}.
\end{align*}
Therefore, we have
\begin{align*}
\Ceval(\Fc)\leq\sqrt{2\cdot\max_{h\in[H],s\in\Sc_h,a\in\Ac}\frac{d^{\pis}_h(s,a)}{d^{\pir}_h(s,a)}}.
\end{align*}

For $\CKL$, we have
\begin{align*}
&\sum_{h=1}^H\Eb_{s_h\sim d^{\pis}_h}\left[\KL(\pis(s_h)\Vert \pir(s_h))\right]=\sum_{h=1}^H\sum_{s\in\Sc_h} d^{\pis}_h(s)\sum_{a\in\Ac}\pis(a|s)\log\frac{\pis(a|s)}{\pir(a|s)}\\
&\qquad=\sum_{h=1}^H\sum_{s\in\Sc_h,a\in\Ac} d^{\pis}_h(s,a)\log\frac{\pis(a|s)}{\pir(a|s)}\\
&\qquad\leq\sum_{h=1}^H\sum_{s\in\Sc_h,a\in\Ac} d^{\pis}_h(s,a)\log\frac{\pis(a|s)}{\pir(a|s)}+\sum_{h=1}^H\sum_{s\in\Sc_h}d^{\pis}_h(s)\log\frac{d^{\pis}_h(s)}{d^{\pir}_h(s)}\\
&\qquad=\sum_{h=1}^H\sum_{s\in\Sc_h,a\in\Ac} d^{\pis}_h(s,a)\log\frac{\pis(a|s)}{\pir(a|s)}+\sum_{h=1}^H\sum_{s\in\Sc_h,a\in\Ac}d^{\pis}_h(s,a)\log\frac{d^{\pis}_h(s)}{d^{\pir}_h(s)}\\
&\qquad=\sum_{h=1}^H\sum_{s\in\Sc_h,a\in\Ac}d^{\pis}_h(s,a)\log\frac{d^{\pis}_h(s,a)}{d^{\pir}_h(s,a)}\leq H\log\left(\max_{h\in[H],s\in\Sc_h,a\in\Ac}\frac{d^{\pis}_h(s,a)}{d^{\pir}_h(s,a)}\right).
\end{align*}

\end{proof}

With Proposition~\ref{prop:relax}, we only need to prove the following theorem to validate Theorem~\ref{thm:rlhf}:

\begin{theorem}
\label{thm:rlhf-re}
Suppose Assumption~\ref{ass:reward},\ref{ass:Q},\ref{ass:reward-conc},\ref{ass:Q-conc},\ref{ass:KL-bound} hold. Then with probability at least $1-\delta$, we have Algorithm~\ref{alg:drpo} with NPG update (Algorithm~\ref{alg:npg}) returns a policy $\hpi$ which satisfies
\begin{align*}
V^{\pis,\rs}(s_1)-V^{\hpi,\rs}(s_1)\leq \left(C_s(T\rmax/\lambda)+C_r(\Rc)\right)\emle + \epmd',
\end{align*}
where
\begin{align*}
&\emle:=O\left(\sqrt{\frac{\kappa^2}{M}\log\frac{|\Rc|}{\delta}}\right),\quad\eeval':= O\left(\sqrt{\frac{\Ceval^2(\Fc) T\rmax^2}{N}\log\frac{T|\Fc|}{\delta}}\right)\\
&\epmd':=\frac{\CKL}{\eta T} + \frac{H\rmax^2\eta}{2} +\lambda \CKL + 2H\eeval'.
\end{align*}
\end{theorem}

In this section we provide the proof of Theorem~\ref{thm:rlhf-re}. Our proof consists of three steps: we first quantify the estimation error of the Q function incurred by LSR oracles -- this step only involves standard supervised learning analysis, then study the performance guarantee of NPG, and lastly investigate how to deal with the reward uncertainty and obtain the final suboptimality gap.

\subsection{Q function Estimation Error}
We have the following lemma to bound the estimation error $\left|\hQ(s,a)-Q^{\pi^{t},\hr}(s,a)\right|$:
\begin{lemma}
	\label{lem:Q-est}
	Fix any $\delta_1\in(0,1]$. With Assumption~\ref{ass:Q}, we have with probability at least $1-\delta_1$ that for all $t\in[T]$,
	\begin{align*}
	\left|\Eb_{h\sim\unif([H]),s\sim d^{\pis}_h,a\sim\bpi(s)}\left[\hQ^t(s,a)-Q^{\pi^t,\hr}(s,a)\right]\right|\leq\Ceval(\Fc)\sqrt{\frac{256 \rmax^2}{N_0}\log\frac{2T|\Fc|}{\delta_1}}:=\eeval',
	\end{align*}
 where $\bpi\in\{\pi^t,\pis\}$.
\end{lemma}
\begin{proof}
From the guarantee of least squares (Lemma~\ref{lem:lsr} in Appendix~\ref{sec:aux}), fix $t\in[T]$, we have with probability at least $1-\delta_1$ that,
\begin{align*}
\Eb_{h\sim\unif([H]),s\sim d^{\pir}_h,a\sim\left(\frac{1}{2}\pir(s)+\frac{1}{2}\pi^t(s)\right)}\left[\left(\hQ^t(s,a)-Q^{\pi^t,\hr}(s,a)\right)^2\right]\leq \frac{256\rmax^2}{N_0}\log\frac{2|\Fc|}{\delta_1}.
\end{align*}
Take union bound over $t\in[T]$ and we have for all $t\in[T]$ that
\begin{align*}
\left|\Eb_{h\sim\unif([H]),s\sim d^{\pis}_h,a\sim\bpi(s)}\left[\hQ^t(s,a)-Q^{\pi^t,\hr}(s,a)\right]\right|\leq\Ceval(\Fc)\sqrt{\frac{256 \rmax^2}{N_0}\log\frac{2T|\Fc|}{\delta_1}}.
\end{align*}
\end{proof}

\subsection{NPG Analysis}
In the following discussion we use $f(s)$ to denote the vector $f(s,\cdot)$ for all functions $f$. We have the following lemma which indicates that NPG is able to find a near optimal policy with respect to the estimated reward $\hr$ (recall that $\eeval$ is defined in Lemma~\ref{lem:Q-est}): 
\begin{lemma}
\label{lem:npg}
Denote the event in Lemma~\ref{lem:Q-est} by $\Ec_1$. Then conditioned on $\Ec_1$, with Assumption~\ref{ass:Q} and \ref{ass:KL-bound}, we have
\begin{align*}
V^{\pis,\hr}(s_1) - V^{\hpi,\hr}(s_1)\leq \frac{\CKL}{\eta T} + \frac{H\rmax^2\eta}{2} + 2H\eeval' +\lambda\CKL:=\epmd'.
\end{align*}
\end{lemma}
\begin{proof}
In the following proof we use $g(\pi(s))$ to denote $\KL(\pi(s)\Vert\pir(s))$ for any policy $\pi$. First note that from the update rule in line 8 of \cref{alg:drpo}, due to first order optimality, we know for all distribution $p\in\Delta(\Ac)$ and all $t\in[T], s\in\Sc$ that: %
\begin{align}
\label{eq:npg-1}
\left\langle -\eta\hQ^t(s)+(1+\eta\lambda)\nabla g(\pi^{t+1}(s)) - \nabla g(\pi^{t}(s)), p - \pi^{t+1}(s)\right\rangle \geq 0 .
\end{align}

This implies that for all $t\in[T], s\in\Sc$, we have
\begin{align*}
&\left\langle \eta\hQ^t(s), \pis(s) - \pi^t(s)\right\rangle + \eta\lambda g(\pi^t(s)) - \eta\lambda g(\pis(s))\\
=&\left\langle \eta\hQ^t(s)-(1+\eta\lambda)\nabla g(\pi^{t+1}(s)) + \nabla g(\pi^{t}(s)), \pis(s) - \pi^{t+1}(s)\right\rangle\\
& + \left\langle\nabla g(\pi^{t+1}(s)) - \nabla g(\pi^t(s)), \pis(s)-\pi^{t+1}(s)\right\rangle + \left\langle \eta\hQ^t(s), \pi^{t+1}(s)-\pi^t(s)\right\rangle\\
&  + \left\langle \eta\lambda\nabla g(\pi^{t+1}(s)), \pis(s)-\pi^{t+1}(s)\right\rangle+ \eta\lambda g(\pi^t(s))-\eta\lambda g(\pis(s))\\
\leq &\underbrace{\left\langle\nabla g(\pi^{t+1}(s)) - \nabla g(\pi^t(s)), \pis(s)-\pi^{t+1}(s)\right\rangle}_{(1)} +\underbrace{\left\langle \eta\hQ^t(s), \pi^{t+1}(s)-\pi^t(s)\right\rangle}_{(2)}\\
&  + \underbrace{\left\langle \eta\lambda\nabla g(\pi^{t+1}(s)), \pis(s)-\pi^{t+1}(s)\right\rangle+ \eta\lambda g(\pi^t(s))-\eta\lambda g(\pis(s))}_{(3)},
\end{align*}
where the last step is due to Equation~\eqref{eq:npg-1}. Now we bound the term (1)(2)(3) respectively.

\paragraph{Bounding term (1).} Note that the KL divergence is indeed the Bregman divergence induced by $g$, therefore the following three point lemma holds true:
\begin{lemma}[three point lemma]
\label{lem:three}
For any distributions $p_1(s),p_2(s),p_3(s)\in\Delta(\Ac)$ ,we have
\begin{align*}
\left\langle \nabla g(p_1(s))-\nabla g(p_2(s)),p_3(s)- p_1(s)\right\rangle = \KL(p_3(s)\Vert p_2(s)) - \KL(p_3(s)\Vert p_1(s)) - \KL(p_1(s) \Vert p_2(s)).
\end{align*}
\end{lemma}
\begin{proof}
From definition of $g$, we know $\nabla g(p(s)) = \log p(s) - \log\pir(s) + \boldsymbol{1}$. This implies that
\begin{align*}
\left\langle \nabla g(p_1(s))-\nabla g(p_2(s)),p_3(s)- p_1(s)\right\rangle = \left\langle \log p_1(s)-\log p_2(s),p_3(s)- p_1(s)\right\rangle.
\end{align*}
Substitute the definition of KL divergence and we can prove the lemma.
\end{proof}
From Lemma~\ref{lem:three}, we can rewrite (1) as follows:
\begin{align*}
(1) = \KL(\pis(s)\Vert \pi^t(s)) - \KL(\pis(s)\Vert \pi^{t+1}(s)) -\KL(\pi^{t+1}(s)\Vert\pi^t(s)).
\end{align*}

\paragraph{Bounding term (2).} From Cauchy-Schwartz inequality, we have
\begin{align*}
(2)\leq \frac{1}{2}\left\Vert\pi^{t+1}(s)-\pi^t(s)\right\Vert_1^2 + \frac{\eta^2}{2}\left\Vert\hQ^t(s)\right\Vert_{\infty}^2\leq \frac{1}{2}\left\Vert\pi^{t+1}(s)-\pi^t(s)\right\Vert_1^2 + \frac{\eta^2\rmax^2}{2}. 
\end{align*}

\paragraph{Bounding term (3).} Since $g$ is convex, we know
\begin{align*}
	\left\langle \eta\lambda\nabla g(\pi^{t+1}(s)), \pis(s)-\pi^{t+1}(s)\right\rangle\leq \eta\lambda g(\pis(s))- \eta\lambda g(\pi^{t+1}(s)).
\end{align*} 
This implies that
\begin{align*}
(3)\leq \eta\lambda\left(g(\pi^t(s))-g(\pi^{t+1}(s))\right).
\end{align*}

In summary, we have for all $t\in[T], s\in\Sc$ that %
\begin{align*}
	&\left\langle \eta\hQ^t(s), \pis(s) - \pi^t(s)\right\rangle + \eta\lambda g(\pi^t(s)) - \eta\lambda g(\pis(s))\\
	\leq& \left(\KL(\pis(s)\Vert \pi^t(s)) - \KL(\pis(s)\Vert \pi^{t+1}(s))\right) + \eta\lambda\left(g(\pi^t(s))-g(\pi^{t+1}(s))\right)\\
	& + \frac{\eta^2}{2}\CQ^2 + \left(\frac{1}{2}\left\Vert\pi^{t+1}(s)-\pi^t(s)\right\Vert_1^2-\KL(\pi^{t+1}(s)\Vert\pi^t(s))\right)\\
	\leq& \left(\KL(\pis(s)\Vert \pi^t(s)) - \KL(\pis(s)\Vert \pi^{t+1}(s))\right) + \eta\lambda\left(g(\pi^t(s))-g(\pi^{t+1}(s))\right) + \frac{\eta^2\rmax^2}{2},
\end{align*}
where the last step is due to Pinsker's inequality.

This implies that
\begin{align}
&\sum_{t=1}^T\sum_{h=1}^H\Eb_{s_h\sim d^{\pis}_h}\left[\left\langle \eta\hQ^t(s_h), \pis(s_h) - \pi^t(s_h)\right\rangle + \eta\lambda g(\pi^t(s_h)) - \eta\lambda g(\pis(s_h))\right]\notag\\
\leq&\sum_{h=1}^H\Eb_{s_h\sim d^{\pis}_h}\left[\KL(\pis(s_h)\Vert \pi^1(s_h)) - \KL(\pis(s_h)\Vert \pi^{T+1}(s_h))\right]\notag\\
& \qquad\qquad\qquad\qquad+ \eta\lambda\sum_{h=1}^H\Eb_{s_h\sim d^{\pis}_h}\left[g(\pi^1(s_h)-g(\pi^{T+1}(s_h)))\right] + \frac{HT\rmax^2\eta^2}{2}\notag\\
\leq&\sum_{h=1}^H\Eb_{s_h\sim d^{\pis}_h}\left[\KL(\pis(s_h)\Vert \pir(s_h))\right]+\frac{HT\rmax^2\eta^2}{2}\leq\CKL + \frac{HT\rmax^2\eta^2}{2}.\label{eq:npg-2}
\end{align}
Note that here we use the fact that we initialize the policy as $\pi^1 = \pir$ and thus $g(\pi^1(s)) = 0$.
On the other hand, note that we have the following performance difference lemma, whose proof is deferred to Appendix~\ref{proof:lem-perf}:
\begin{lemma}[performance difference lemma]
\label{lem:perf}
For any policy $\pi,\pi'$ and reward function $r$, we have:
\begin{align*}
V^{\pi,r}(s_1) - V^{\pi',r}(s_1) = \sum_{h=1}^H \Eb_{s_h\sim d^{\pi}_h}\left[\left\langle Q^{\pi',r}(s_h), \pi(s_h)-\pi'(s_h)\right\rangle\right]. 
\end{align*}
\end{lemma}

Now substitute Lemma~\ref{lem:perf} into Equation~\eqref{eq:npg-2}, and from Lemma~\ref{lem:Q-est} we have
\begin{align*}
\frac{1}{T}\sum_{t=1}^T\left(V^{\pis,\hr}(s_1) - V^{\pi^t,\hr}(s_1)\right)\leq \frac{\CKL}{\eta T} + \frac{H\rmax^2\eta}{2} + 2H\eeval'+\lambda\CKL.
\end{align*}
This is equivalent to
\begin{align*}
V^{\pis,\hr}(s_1) - V^{\hpi,\hr}(s_1)\leq \frac{\CKL}{\eta T} + \frac{H\rmax^2\eta}{2} + 2H\eeval' +\lambda\CKL,
\end{align*}
which concludes our proof.
\end{proof}

We also would like to bound the KL divergence between $\hpi$ and $\pir$ as shown in the following lemma:
\begin{lemma}
\label{lem:npg-kl}
We have for all $t\in[T],s\in\Sc$ that
\begin{align*}
\KL(\pi^t(s)\Vert\pir(s))\leq\frac{\rmax(t-1)}{\lambda}.
\end{align*}
\end{lemma}
\begin{proof}
From the NPG update and use the fact that $\pi^{t+1}$ is the minimizer, we know for all $t\in[T], s\in\Sc$:
\begin{align*}
\KL(\pi^{t+1}(s)\Vert\pir(s))-\KL(\pi^{t}(s)\Vert\pir(s))\leq\frac{1}{\lambda}\left\langle \hQ^t(s), \pi^{t+1}(s)-\pi^t(s)\right\rangle\leq\frac{\rmax}{\lambda},
\end{align*}
where we utilize Assumption~\ref{ass:Q} in the second step. 

Note that $\KL(\pi^{1}(s)\Vert\pir(s))=0$ since $\pi^1=\pir$. This implies that for all $t\in[T]$:
\begin{align*}
\KL(\pi^{t}(s)\Vert\pir(s))
\leq\frac{\rmax(t-1)}{\lambda}.
\end{align*}
\end{proof}

\subsection{Unregularized Suboptimality Gap w.r.t. $\rs$}
Now we can start to prove Theorem~\ref{thm:rlhf}. First we have
\begin{align*}
	&V^{\pis,\rs}(s_1)-V^{\hpi,\rs}(s_1)=\left(V^{\pis,\rs}(s_1)-V^{\pis,\hr}(s_1)\right) + \left(V^{\pis,\hr}(s_1)-V^{\hpi,\hr}(s_1)\right) + \left(V^{\hpi,\hr}(s_1)-V^{\hpi,\rs}(s_1)\right)\\
	&\qquad=\underbrace{\left(\Eb_{\tau\sim d^{\pis}}\left[\rs(\tau)-\hr(\tau)\right]-\Eb_{\tau\sim d^{\pir}}\left[\rs(\tau)-\hr(\tau)\right]\right)}_{(1)} + \underbrace{\left(V^{\pis,\hr}(s_1)-V^{\hpi,\hr}(s_1)\right)}_{(2)} \\
	&\qquad\quad+ \underbrace{\frac{1}{T}\sum_{t=1}^T\left(\Eb_{\tau\sim d^{\pir}}\left[\rs(\tau)-\hr(\tau)\right]-\Eb_{\tau\sim d^{\pi^t}}\left[\rs(\tau)-\hr(\tau)\right]\right)}_{(3)}.
\end{align*}
Next we bound term (1)(2)(3) respectviely.

\paragraph{Bounding term (1).} 
From the guarantee of MLE (Lemma~\ref{lem:mle} in Appendix~\ref{sec:aux}) we have with probability at least $1-\delta_2$ that
\begin{align}
	\label{eq:rlhf-2}
	\Eb_{\tau^0\sim d^{\pir},\tau^1\sim d^{\pir}}\left[\left|\rs(\tau_0)-\rs(\tau_1) -\hr(\tau_0)+\hr(\tau_1)\right|^2\right]\leq \frac{c_1\kappa^2\log(|\Rc|/\delta_2)}{M}:=\emle^2,
\end{align}
where $c_1>0$ is a universal constant. Denote the event of the above inequality by $\Ec_2$. Then conditioned on $\Ec_2$, we have
\begin{align*}
	&(1)= \Eb_{\tau^0 \sim d^{\pis},\tau^1 \sim d^{\pir}}[\rs(\tau^0)-\rs(\tau^1)-r(\tau^0)+r(\tau^1)]\notag\\
	&\qquad\leq C_r(\Rc)\sqrt{\Eb_{\tau^0\sim d^{\pir},\tau^1\sim d^{\pir}}\left[\left|\rs(\tau_0)-\rs(\tau_1) -\hr(\tau_0)+\hr(\tau_1)\right|^2\right]}\notag\\
	&\qquad\leq C_r(\Rc)\emle.
\end{align*}

\paragraph{Bounding term (2).} From Lemma~\ref{lem:npg}, conditioned on $\Ec_1$, we have
\begin{align*}
	V^{\pis,\hr}(s_1) - V^{\hpi,\hr}(s_1) \leq \epmd'.
\end{align*}

\paragraph{Bounding term (3).} 

Note that from Lemma~\ref{lem:npg-kl}, we have $\pi^t\in\Theta(\pir,(t-1)\rmax/\lambda)$ for all $t\in[T]$. Therefore, following the same arguments as we have to bound term (1), we have for all $t\in[T]$,
\begin{align*}
	\Eb_{\tau\sim d^{\pir}}\left[\rs(\tau)-\hr(\tau)\right]-\Eb_{\tau\sim d^{\pi^t}}\left[\rs(\tau)-\hr(\tau)\right] \leq C_s((t-1)\rmax/\lambda)\emle,
\end{align*}
which implies that
\begin{align*}
(3)\leq C_s(T\rmax/\lambda)\emle
\end{align*}

Overall, we have conditioned on event $\Ec_1\cap\Ec_2$,
\begin{align*}
	V^{\pis,\rs}(s_1)-V^{\hpi,\rs}(s_1)\leq (C_r(\Rc)+C_s(T\rmax/\lambda))\emle + \epmd'.
\end{align*}
We finish the proof by letting $\delta_1=\delta_2=\delta/2$.

\section{Auxiliary Lemmas}
\label{sec:aux}
\subsection{Least Sqaures Guarantee}
\begin{lemma}[Lemma 15 in \citet{song2022hybrid}]
\label{lem:lsr}
Fix any \(R > 0\), \(\delta \in (0, 1)\) and assume we have a class of real valued functions \(\Hc: \Xc \mapsto [-R, R]\). Suppose we have $K$ i.i.d. samples $\{(x_k,y_k)\}_{k=1}^K$ where $x_k\sim\rho$ and $y_k$ is sampled via the conditional probability $p(\cdot \mid x_k)$:
\begin{align*}
y_k \sim  p(\cdot \mid x_k) := h^*(x_k) + \epsilon_k, 
\end{align*} where $h^*\in\Hc$ and $\{\epsilon_k\}_{k=1}^K$  are independent random variables such that $\Eb[y_k \mid x_k] = h^{*}(x_k)$. Additionally, suppose that \(\max_k |y_k| \leq R\) and \(\max_{x} |{h^*(x)}| \leq R\). Then the least square solution $\widehat{h} \leftarrow \argmin_{h \in \Hc} \sum_{k=1}^K \left(h(x_k) - y_k\right)^2$ satisfies with probability at least $1 - \delta$, 
\begin{align*}
 \Eb_{x \sim \rho} \left[\left(\widehat{h}(x) - h^*(x)\right)^2\right] &\leq \frac{256 R^2 \log(2 |\Hc|/\delta)}{K}. 
\end{align*}
\end{lemma}
The proof is the same as in \citet{song2022hybrid} and thus is omitted here.

\subsection{Maximum Likelihood Estimation Guarantee}
\begin{lemma}
\label{lem:mle}
With Assumption~\ref{ass:reward}, we have with probability at least $1-\delta$ that
\begin{align*}
	\Eb_{\tau^0\sim d^{\pir},\tau^1\sim d^{\pir}}\left[\left|\rs(\tau_0)-\rs(\tau_1) -\hr(\tau_0)+\hr(\tau_1)\right|^2\right]\leq \frac{c_1\kappa^2\log(|\Rc|/\delta)}{M},
\end{align*}
where $c_1>0$ is a universal constant.
\end{lemma}
\begin{proof}
The proof largely follows the proof of Theorem 1 in \citet{zhan2023provableoff}. Specifically, we have the following lemma from \cite{zhan2023provableoff}:
\begin{lemma}[Lemma 2 in \citet{zhan2023provableoff}]
\label{lem:mle_aux}
Fix any $\delta\in(0,1]$. Then with probability at least $1-\delta$, we have that for all reward function $r\in\Rc$,
\begin{align*}
\Eb_{\tau^0,\tau^1\sim d^{\pir}}\bigg[\Big\Vert P_{r}(\cdot|\tau^0,\tau^1)-P_{\rs}(\cdot|\tau^0,\tau^1)\Big\Vert_1^2\bigg]\leq \frac{c_1}{M}\bigg(\sum_{m=1}^M\log\bigg(\frac{P_{\rs}(o^{m}|\tau^{m,0},\tau^{m,1})}{P_{r}(o^{m}|\tau^{m,0},\tau^{m,1})}\bigg)+\log\frac{|\Rc|}{\delta}\bigg).
\end{align*}
\end{lemma}
Then from Lemma~\ref{lem:mle_aux}, since $\sum_{m=1}^M\log P_{\rs}(o^{m}|\tau^{m,0},\tau^{m,1})\leq \sum_{m=1}^M\log P_{\hr}(o^{m}|\tau^{m,0},\tau^{m,1})$, we have with probability at least $1-\delta$:
\begin{align}
\label{eq:known-1}
\Eb_{\tau^{0},\tau^1\sim d^{\pir}}\bigg[\Big\Vert P_{\hr}(\cdot|\tau^0,\tau^1)-P_{\rs}(\cdot|\tau^0,\tau^1)\Big\Vert_1^2\bigg]\leq\frac{c_1\log\frac{|\Rc|}{\delta}}{M}.
\end{align}

Then under Assumption~\ref{ass:reward}, we can apply the mean value theorem between $r^{\star}(\tau_1)-r^{\star}(\tau_0)$ and $\hr(\tau_1)-\hr(\tau_0)$ to \eqref{eq:known-1} and ensure that
\begin{align*}
    \Eb_{\tau^{0},\tau^1\sim d^{\pir}}[|(r^{\star}(\tau_1)-r^{\star}(\tau_0)) - (\hr(\tau_1)-\hr(\tau_0))|^2]\leq \frac{c_1\kappa^2\log\frac{|\Rc|}{\delta}}{M},
\end{align*}
where $\kappa:=\frac{1}{\inf_{x\in[-\rmax,\rmax]}\Phi'(x)}$ measures the non-linearity of the link function $\Phi$.
\end{proof}

\subsection{Performance Difference}
\label{proof:lem-perf}
We restate and prove Lemma~\ref{lem:perf} as follows.
\begin{lemma}
\label{lem:perf_re}
For any policy $\pi,\pi'$ and reward function $r$, we have:
\begin{align*}
V^{\pi,r}(s_1) - V^{\pi',r}(s_1) = \sum_{h=1}^H \Eb_{s_h\sim d^{\pi}_h}\left[\left\langle Q^{\pi',r}(s_h), \pi(s_h)-\pi'(s_h)\right\rangle\right].
\end{align*}
\end{lemma}
\begin{proof}
For any two policies $\pi,\pi'$ and reward $r$, we have that
\begin{align*}
	&V^{\pi,r}(s_1) - V^{\pi',r}(s_1)\\
	=&\mathbb{E}_{\pi}\left[r(s_1,a_1)+V^{\pi,r}(s_2)\right]-\Eb_{\pi}\left[V^{\pi',r}(s_1)\right]\\
	=&\mathbb{E}_{\pi}\left[\left(Q^{\pi',r}(s_1,a_1)-V^{\pi',r}(s_2)\right)+V^{\pi,r}(s_2)\right] - \Eb_{\pi}\left[V^{\pi',r}(s_1)\right]\\
	=&\mathbb{E}_{\pi}\left[V^{\pi,r}(s_2)- V^{\pi',r}(s_2)\right] + \mathbb{E}_{\pi}\left[Q^{\pi',r}(s_1,a_1)- V^{\pi',r}(s_1) \right]\\
	=&\mathbb{E}_{\pi}\left[V^{\pi,r}(s_2)- V^{\pi',r}(s_2)\right] + \mathbb{E}_{s_1\sim d^{\pi}_1}\left[\left\langle Q^{\pi',r}(s_1,\cdot), \pi(\cdot|s_1)-\pi'(\cdot|s_1)\right\rangle\right]\\
	=&\cdots=\sum_{h=1}^H \Eb_{s_h\sim d^{\pi}_h}\left[\left\langle Q^{\pi',r}(s_h), \pi(s_h)-\pi'(s_h)\right\rangle\right].
\end{align*}
This concludes our proof.
\end{proof}

\section{Additional Experiment Details}
\label{app:details}
\subsection{Experiment Hyperparameters and Task Details}

\subsubsection{Task Details}

\label{appendix:task}

We present dataset specific details in table \ref{tbl:dataset}

\begin{table}
    \centering 
    \begin{tabular}{c|ccc}
        \toprule
        Task & Train/Val/Test & Prompt & Gen. Length\\
        \midrule
        \texttt{TL;DR} & 117K/6.45K/6.55K & "TL;DR: " & 53\\
        \texttt{CNN/DailyMail} & 287K/13.4K/11.4K & "TL;DR: " & 64\\
        \bottomrule
    \end{tabular}
    \vspace{0.5mm}
    \caption{Train, val, test splits, prompts, and max generation length used for each task.}
    \label{tbl:dataset}
\end{table}

For both datasets we obtained the training data from \url{https://github.com/openai/summarize-from-feedback}.

\subsection{Dataset Reset Implementation Details}
\label{ssec:reset_pseudocode}
Here is a code snippet of the logit processor that handles dataset resets from references for a HuggingFace transformers model. $\beta$ here represents the proportion of generations in the batch to do resets for.

\begin{verbatim}
import torch
import numpy as np
from transformers import LogitsProcessor

class ResetProcessor(LogitsProcessor):
    def __init__(self, references, beta, rng, seq_lens):
        self.counter = 0
        self.references = references
        self.seq_lens = seq_lens
        self.create_mask(beta, rng)

    def create_mask(self, beta, rng):
        batch_size, seq_len = self.references.shape[:2]
        # Mixin
        init_mask = rng.choice(
            [True, False], size=(batch_size, 1), p=[beta, 1-beta]
        )
        init_mask = np.tile(init_mask, (1, seq_len))
        # Rollin Selection
        length_masks = np.tril(np.ones((seq_len, seq_len)))
        masks = []
        for length in self.seq_lens:
            if length < 2:
                masks.append(np.zeros((seq_len)).astype(bool))
            else:
                masks.append(
                    rng.choice(length_masks[: length - 1, :]
                ).astype(bool))

        self.rollin_mask = np.stack(masks)
        self.rollin_mask[~init_mask] = False

    def __call__(
        self, input_ids: torch.LongTensor, scores: torch.FloatTensor
    ) -> torch.FloatTensor:
        vocab_size = scores.size(-1)
        new_scores = one_hot(
            self.references[:, self.counter], num_classes=vocab_size
        ).float()
        new_scores[new_scores == 0] = -float("inf")
        mask = self.rollin_mask[:, self.counter]
        assert scores.shape == new_scores.shape

        new_scores = new_scores.to(scores.device)
        # Only do Teacher Forcing on the rollins
        scores[mask] = new_scores[mask]
        self.counter += 1
        return scores
\end{verbatim}

\subsubsection{Computation}
Note since we start with the references from the dataset, the computational requirements to generate with resets are the same as generating from the initial state distribution. For all of our experiments, we ran with the same per device batch size between PPO and DR-PO. For this work, we made use of 16 A6000 gpus with 48GB of VRAM. We used 4 gpus for each run.

\subsection{Details on GPT4 Winrate}
\label{app:winrate}
For winrate calculation, we used the following prompt:

\begin{verbatim}
Which of the following summaries does a better job of summarizing the most important points
in the given forum post, without including unimportant or irrelevant details? Judge based
on accuracy, coverage, and coherence.

### Post:

{{post}}

### Summary A:

{{summarya}}

### Summary B:

{{summaryb}}

### Instructions:
FIRST provide a one-sentence comparison of the two summaries, explaining which you prefer
and why. SECOND, on a new line, state only "A" or "B" to indicate your choice. Your response
should use the format:
Comparison: <one-sentence comparison and explanation>
Preferred: <"A" or "B">
\end{verbatim}

\subsubsection{Win Rate Example}
Here is an example of getting a one sentence explanation as to why GPT4 chose certain generations for the winrate.

\textbf{Prompt}

\begin{verbatim}
SUBREDDIT: r/AskReddit

TITLE: How do you get someone out of your head?

POST: Hi,
I'm 22, and I have been with my girlfriend for 5 years now. We recently moved together.
We've always loved each other intensely.

Problem, I recently started to have feelings for an other person (a friend). This person
has had a boyfriend for now 3 years, and has absolutely no ideas. Those feelings were 
so strong, it was hard to hide them. After 2 months of me being distant and really sad, 
my girlfriend forced me to say what was bothering me. I'm not a good liar, and now she knows.

We decided to give us a week alone, I went to my parents. 

Now, I'm completely lost. I keep on thinking about this person, and I hate that. I would 
like for those feelings to go away, to leave me alone. But I can't.  

What do I do? It's been 3 months now, and I'm just desperate.

TL;DR:
\end{lstlisting}

\textbf{DR-PO Generation (Summary A)}
\begin{verbatim}
I recently started to have feelings for someone else, my girlfriend knows, we decided to 
give ourselves a week alone, now I'm completely lost, I hate that, what do I do
\end{verbatim}

\textbf{Reference (Summary B)}
\begin{verbatim}
long relationship; fell in love with an other person; admitted it; would like it to 
disappear, though it doesn't.
\end{verbatim}

\textbf{GPT4 Explanation for Choosing DR-PO}
\begin{verbatim}
Summary A is more detailed and better represents the tone and content of the forum post, 
while Summary B is too abbreviated and omits important details such as the fact 
that the poster's girlfriend now knows about their feelings.
\end{verbatim}

\subsection{Examples from Test}
\textbf{Prompt 1}
\begin{verbatim}
SUBREDDIT: r/AskReddit

TITLE: My friend works at a California Charter School and their principal is disregarding
their charter and threatening to fire anyone who questions his leadership. What can she do?

POST: I'm not an educator but my friend is having a really hard time at her job right now 
and I was wondering if anyone could help.

**The Issue:**

She works for a charter school in California where the principal has gone off the deep end. 
At first it started casually with cuts and a few firings (understandable at the time because 
of budget cuts). It steadily got worse however as he slowly stripped the teachers of any 
rights/protections they once had. Now he's threatening to fire them at any time if they 
refuse to work late/take personal days/attempt talks with the teachers union/show signs of 
disagreeing with the principal's leadership, etc.

This guy has continually canceled meetings by any form of teacher-staffed governing boards 
that would normally provide the teachers with a voice. Now he just dictates policy for the 
school with no regard or input from the teachers. He has threatened to let anyone go if they 
show even the slightest dissatisfaction with his decisions.

Keep in mind this was a very good school before this guy came in. It's not like it was a 
failing school in need of some tough love. As far as scholastic achievement goes, not much 
has changed since this guy was hired on (which is why I suspect he still has a job despite 
complaints to district).

Is there anything my friend can do to oust this guy or at least protect her job?

TL;DR:
\end{verbatim}

\textbf{DR-PO}
\begin{verbatim}
 My friend works at a California charter school where the principal is disregarding their 
 charter and threatening to fire anyone who questions his leadership. What can she do?
\end{verbatim}

\textbf{Prompt 2}
\begin{verbatim}
SUBREDDIT: r/AskReddit

TITLE: I am in the financial and insurance industry, but is looking to go back to school 
to become a vet.  Any suggessions?

POST: I'm not 100%
being a Veternarian is something I really want to do in life.  
A bit of background. I'm currently 24, graduated with a finance degree 3 years ago. 
I'm currently a financial rep focusing on selling insurance.  Even though I don't hate 
my job I feel like it doesn't fit my personality.  But I stuck with it for a while 
because I don't like to quit easily. The idea of becoming a vet happened when a 
chinchilla of mine passed because I didn't take it to the vet in time, and I really love 
animals. People tell me I'm a warm person and great with kids and animals. I feel like 
this is something I came up with myself and not what other want me to be.
     
So back to reality, I did a tiny bit of research on this. The closest grad school that has 
this program is Cornell university (I'm located in NY) I'd probably need to do an undergrad 
in science or medical field.  I'm a little unsure of 8 more years of school, but I guess if 
there's no other choice.  Another concern is money, I only have about 10k in bank, I cannot 
touch my retirement and life insurance money. I'm not sure if I can qualify for any federal 
grant. I may also need to move back with my parents. One idea I have is to just brokage 
product with high residuals while I wait to get into a school. I'll also appraciate any 
insight and experiences from a vet or a person going back to school.

TL;DR:
\end{verbatim}

\textbf{DR-PO}
\begin{verbatim}
 I'm in the financial and insurance industry, but am looking to go back to school to become 
 a vet. Any suggestions or experiences from a vet or someone going back to school.
\end{verbatim}
\newpage

\subsection{Hyperparameters}
We write the relevant hyperparameters from our experiments for DPO, PPO, SFT, and DRPO in table \ref{tbl:ppo_hparams}.
\begin{table}[h]
   \vspace{-2mm}
   \centering
   \resizebox{0.8\textwidth}{!}{
   \begin{tabular}{ll}
       \toprule
       Setting & Values \\
       \midrule
       model & Pythia 2.8B (HuggingFace Model Card: \url{EleutherAI/pythia-2.8b-deduped})\\
       \midrule
       PPO & train epochs: 1 \\
              & batch size: 512 \\
              & num epochs: 4 \\
              & num minibatches: 1 \\
              & learning rate: 3e-6 \\
              & schedule: linear decay \\
              & discount factor: 1 \\
              & gae $\lambda$: 0.95 \\
              & clip ratio: 0.2 \\
              & value function coeff: 0.1\\
              & kl coefficient: 0.05 \\
       \midrule 
       DR-PO & mixing parameter ($\beta$): 1 \\
       \midrule
          DPO & batch size: 64 \\
              & $\beta$: 0.05 \\
              & learning rate: 3e-6 \\
              & schedule: linear decay \\
              & num train epochs: 1 \\
        \midrule
       Reward Model   & batch size: 64 \\
             & learning rate: 3e-6 \\
             & schedule: linear decay \\
             & num train epochs: 1 \\
       \midrule
       SFT   & batch size: 64 \\
             & learning rate: 3e-6 \\
             & schedule: linear decay \\
             & num train epochs: 1 \\
       \midrule
       LoRA Adapter Config & r: 1024\\
                        & $\alpha$: 2048 \\
                        & dropout: 0.0 \\
                        & bias: False \\
       \midrule
       Decoding & sampling: true \\
                & top k: 0.0 \\
                & top p: 1.0 \\
                & min length: 53 \\
                & max new tokens: 53\\
                & temperature: 0.1 \\
       \midrule
       Tokenizer & padding side: left \\
                 & truncation side: left \\
                 & max length: 563\\
       \bottomrule
   \end{tabular}}
   \caption{Hyperparameters used for TL;DR and CNN/DailyMail. Note that DP-RO and PPO share the same parameters (other than mixing proportion). All processes use the same decoding, LoRA config, and tokenizer parameters.}
   \label{tbl:ppo_hparams}
\end{table}

\newpage
\subsection{Additional Experiments}

\begin{table}[tb]
    \centering
    \begin{tabular}{@{}lcc@{}}
        \toprule
        \textbf{Algorithms} & RM TL;DR Accuracy & RM CNN/DM Accuracy\\
        \midrule
        \texttt{RM}          & \textbf{66.21\%} & \textbf{67.48\%} \\
        \texttt{DPO}         & 65.92\% & 67.28\% \\
        \midrule
        \texttt{RM w/ LoRA}  & 62.87\% & \textbf{66.75\%} \\
        \texttt{DPO w/ LoRA} & \textbf{66.14\%} & 61.78\% \\
        \bottomrule 
    \end{tabular}%
    \caption{\textbf{Reward Model Transfer to CNN/DM:} The accuracy of the RM and DPO's implicit learned reward in accuracy predicting the preference. We evaluate models trained with and without LoRA on TL;DR. We also report the zero-shot performance of these models on the CNN/DailyMail preference dataset from \citet{stiennon2020learning}.}
    \label{tbl:rm_transfer}
\end{table}

Shown in \cref{tbl:rm_transfer}, we investigate DPO's implicit learned reward accuracy to our RM's accuracy on both TL;DR and CNN/DailyMail's test sets. Furthermore, we also report the effects of LoRA on the RM and DPO performance. We see that DPO without LoRA has comparable preference accuracy on CNN/DM as our RM. Thus, we used the DPO policy without LoRA when comparing against PPO and DR-PO in \cref{tbl:cnn_transfer}.

\end{document}